\documentclass{article}

\usepackage{microtype}
\usepackage{graphicx}
\usepackage{multirow}
\usepackage{subfigure}
\usepackage{booktabs} % for professional tables

\usepackage{hyperref}

\usepackage[accepted]{icml2023}

\usepackage{amsmath}
\usepackage{amssymb}
\usepackage{mathtools}
\usepackage{amsthm}
\usepackage{enumitem}
\usepackage{breqn}
\usepackage{dsfont}
\usepackage{bbding}
\usepackage{xcolor,colortbl}

\DeclareMathOperator*{\argmin}{arg\,min}
\DeclareMathOperator*{\argmax}{arg\,max}

\usepackage[capitalize,noabbrev]{cleveref}

\theoremstyle{plain}
\newtheorem{theorem}{Theorem}[section]

\newtheorem{lemma}[theorem]{Lemma}

\theoremstyle{definition}
\newtheorem{definition}[theorem]{Definition}
\newtheorem{assumption}[theorem]{Assumption}
\theoremstyle{remark}

\usepackage[textsize=tiny]{todonotes}

\icmltitlerunning{UMD: Unsupervised Model Detection for X2X Backdoor Attacks}

\begin{document}
	
\twocolumn[
\vspace{-0.185in}
\icmltitle{UMD: Unsupervised Model Detection for X2X Backdoor Attacks}




\begin{icmlauthorlist}
\vspace{-0.145in}
	\icmlauthor{Zhen Xiang}{uiuc}
	\icmlauthor{Zidi Xiong}{uiuc}
	\icmlauthor{Bo Li}{uiuc}
\vspace{-0.1in}
\end{icmlauthorlist}

\icmlaffiliation{uiuc}{University of Illinois at Urbana-Champaign}

\icmlcorrespondingauthor{Zhen Xiang}{zxiangaa@illinois.edu}
\icmlcorrespondingauthor{Bo Li}{lbo@illinois.edu}

\icmlkeywords{Machine Learning, ICML}

\vskip 0.3in
]



\printAffiliationsAndNotice{}  % leave blank if no need to mention equal contribution

\begin{abstract}
\vspace{-0.025in}
Backdoor (Trojan) attack is a common threat to deep neural networks, where samples from one or more {\it source classes} embedded with a backdoor trigger will be misclassified to adversarial {\it target classes}.
Existing methods for detecting whether a classifier is backdoor attacked are mostly designed for attacks with a single adversarial target (e.g., all-to-one attack).
To the best of our knowledge, without supervision, no existing methods can effectively address the more general X2X attack with an arbitrary number of source classes, each paired with an arbitrary target class.
In this paper, we propose UMD, the {\it first} {\underline{U}nsupervised} {\underline{M}odel} {\underline{D}etection} method that effectively detects X2X backdoor attacks via a joint inference of the adversarial (source, target) class pairs.
In particular, we first define a novel {\it transferability} statistic to measure and select a subset of putative backdoor class pairs based on a proposed clustering approach.
Then, these selected class pairs are jointly assessed based on an aggregation of their reverse-engineered trigger size for detection inference, using a robust and unsupervised anomaly detector we proposed.
We conduct comprehensive evaluations on CIFAR-10, GTSRB, and Imagenette dataset, and show that our \textit{unsupervised} UMD outperforms SOTA detectors (even with supervision) by 17\%, 4\%, and 8\%, respectively, in terms of the detection accuracy against diverse X2X attacks.
We also show the strong detection performance of UMD against several strong adaptive attacks.
\vspace{-0.175in}
\end{abstract}

\section{Introduction}
\label{sec:introduction}
\vspace{-0.025in}

Despite the success of deep neural networks in many applications, they are vulnerable to adversarial attacks such as backdoor (Trojan) attacks \cite{Review, BAsurvey}.
A classical backdoor attack is usually specified by one or more source classes, a target class, and a backdoor trigger, such that test samples from any source class embedded with the trigger will be misclassified to the target class; while samples without the trigger will be correctly classified \cite{BadNet}.
Typically, a backdoor attack is launched by poisoning the training set of the classifier \cite{Targeted, Clean_Label_BA, Haoti, Liu2020Refool, nguyen2021wanet, li_ISSBA_2021}.

\begin{figure}[t]
\centering
\includegraphics[width=.92\columnwidth]{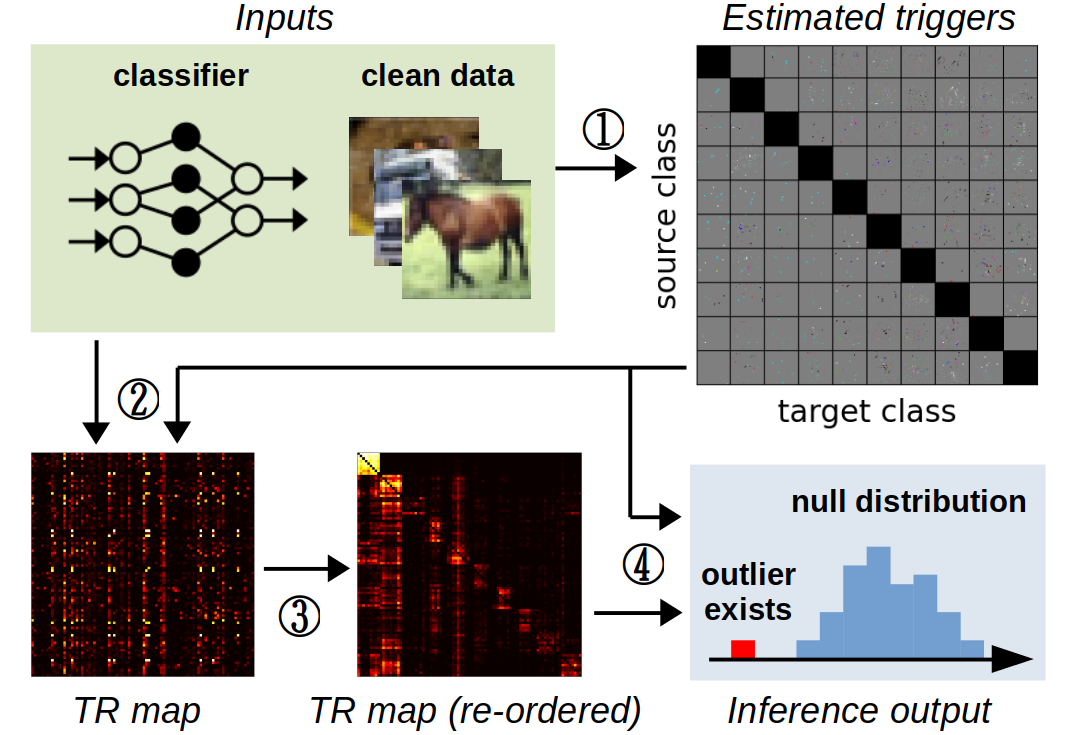}
\vspace{-0.05in}
\caption{Outline of UMD: \textcircled{1} reverse-engineer a trigger for each class pair (Sec. \ref{subsubsec:tr_def}); \textcircled{2} compute TR for all ordered pair of class pairs (Sec. \ref{subsubsec:tr_def}); \textcircled{3} select a subset of putative backdoor class pairs based on TR (Sec. \ref{subsubsec:select_pair}); \textcircled{4} inspect the selected pairs by unsupervised anomaly detection on trigger statistics (Sec. \ref{subsubsec:detection_inference}).}\label{fig:outline}
\vspace{-0.15in}
\end{figure}

Recently, many approaches have been proposed to detect whether a trained classifier is backdoor attacked {\it without access to} the training set or any benign models for supervision (e.g. to set a detection threshold) \cite{DeepInspect, Tabor, Post-TNNLS, DataLimited, NC_blackbox, Competition}.
These methods mainly fall into either a family of reverse-engineering-based detectors (\textbf{RED}s) or a category of meta-classification-based detectors (\textbf{MCD}s).
Typically, REDs reverse-engineer putative triggers for anomaly detection \cite{NC, ABS}, while MCDs train a binary meta classifier on a large number of shadow classifiers with and without attack for detection \cite{META}.
These detectors are effective against the classical backdoor attack, but they may be bypassed by more advanced backdoor attacks recently proposed to defeat them \cite{nguyen2020inputaware, label_smoothed}.

In this paper, we consider \textit{X2X backdoor attacks}, which refer to a broad family of backdoor attacks with an arbitrary number of source classes, each assigned with an arbitrary target class.
Thus, the X2X attack includes many popular attack types such as the ``all-to-one'' attack, ``X-to-one'' attack \cite{Shen2021BackdoorSF}, ``one-to-one'' attack \cite{SS}, and ``all-to-all'' attack \cite{BadNet}.
Unlike other advanced attacks that mostly rely on additional assumptions such as full control of the training process \cite{Zhao_2022_CVPR, Wang_2022_CVPR}, X2X attacks can be easily launched by poisoning the training set just like classical backdoor attacks.
Moreover, to the best of our knowledge, X2X attacks are {\it not detectable} by existing methods -- REDs are mostly designed for ``all-to-one'' attacks only, while MCDs need to assume access to the attack setting to train the shadow classifiers.

To bridge this gap, we propose an {\underline{U}nsupervised} {\underline{M}odel} {\underline{D}etection} approach (UMD) to detect X2X attacks and infer all the class pairs involved in the attack, without any assumptions about the number of source classes or the target class assignment rules.
UMD first reverse-engineers a putative trigger for each class pair. 
Unlike existing detectors that directly use trigger statistics (e.g., the perturbation size of the triggers) for anomaly detection, we calculate a {\it transferability} (TR) statistic for each ordered pair of class pairs.
The TR statistics are then used to select a subset of class pairs that are most likely involved in an X2X attack by solving a novel clustering problem.
Finally, an {\it unsupervised}, bias-reduced anomaly detector is designed to robustly assess the atypicality of the trigger statistic aggregated over the selected class pairs -- these class pairs are deemed the backdoor class pairs when an attack is detected.
Our contributions in this paper are summarized as the following:
\vspace{-0.1in}
\begin{itemize}[leftmargin=*]
\setlength\itemsep{-0.25em}
\item We propose UMD, the \textit{first} unsupervised model detector against X2X backdoor attacks with arbitrary numbers of source classes and arbitrary target class assignments.
\item We propose a statistic -- TR -- to identify backdoor class pairs. In particular, we prove that in ideal cases, TR from one backdoor class pair to another backdoor class pair is guaranteed to be no less than TR from a backdoor class pair to a non-backdoor class pair.
\item We propose a two-step inference procedure for UMD. First, a set of putative backdoor class pairs is selected based on the TR statistics by solving a novel clustering problem using an agglomerative algorithm.
Second, an aggregated trigger statistic based on these selected class pairs is evaluated for inference using our robust, unsupervised anomaly detector, with a confidence threshold adapted to the number of classes in the domain.
\item We conduct extensive experiments to show the strong detection capability of UMD against diverse X2X attacks and several strong advanced adaptive attacks.
We show that our unsupervised UMD outperforms SOTA baselines, including the ones with supervision by \citeauthor{ABS} (2019) and \citeauthor{Shen2021BackdoorSF} (2021) by 17\%, 4\%, and 8\% in the average model inference accuracy against various X2X attacks on CIFAR-10, GTSRB, and Imagenette, respectively.
\end{itemize}

\section{Related Work}\label{sec:related_work}

{\bf Backdoor attacks} While we focus on image classification 
in this paper like most existing works, backdoor attacks have also been extended to other data domains and/or learning paradigms \cite{language_backdoor, badnl, li2022few, Xie2020DBA, Yao_2019_CCS, jia2022badencoder}. For the image domain, apart from the X2X attack focused on in this paper, advanced backdoor attacks, such as clean-label attacks \cite{Clean_Label_BA, Hidden-trigger}, invisible-trigger attacks \cite{Haoti, nguyen2021wanet, Zhao_2022_CVPR, Wang_2022_CVPR}, and physical attacks \cite{Liu2020Refool}, are also proposed to achieve better stealthiness against possible human inspection of either the training set or test instances. Moreover, some advanced backdoor attacks are proposed, e.g., by \citeauthor{nguyen2020inputaware} (2020), \citeauthor{li_ISSBA_2021} (2021), and \citeauthor{xue2022imperceptible} (2022), to evade particular backdoor defenses. In addition to the X2X attack, we will show the effectiveness of our UMD against some of these advanced attacks (including their X2X extensions) as well.

{\bf Backdoor model detection} Existing methods that detect whether a trained classifier is backdoor attacked mainly fall into two categories. Reverse-engineering-based detectors (REDs) estimate putative triggers for anomaly detection \cite{NC, DeepInspect, Post-TNNLS, DataLimited, Shen2021BackdoorSF, taog2022better, hu2022trigger}. Meta-classification-based detectors (MCDs) train a meta classifier using shadow classifiers trained with and without attacks \cite{META, meta_sup}.
Unlike our UMD, these methods {\it cannot} effectively detect X2X attacks since they all rely on assumptions about the target class assignment.
Except for our UMD, several prior model detectors also involved the concept of ``\textit{transferability}''.
For example, transferability is defined at the instance level by \citeauthor{TwoClass} (2022) and \citeauthor{top} (2021), or for each putative target class to be inspected by \citeauthor{ABS} (2019).
Differently, the TR statistic used by our UMD is defined for each {\it ordered pair of class pairs}, which enables UMD to identify backdoor class pairs regardless of the target class assignment.

{\bf Other types of backdoor defense} Backdoor mitigation approaches aim to remove the learned backdoor mapping from a trained classifier \cite{FP, ANP, NAD, ShapPruning, CLP, zeng2022adversarial}.
They usually require a large number of samples and may degrade the clean accuracy of the classifier.
Training-phase defenses aim to obtain a backdoor-free classifier from the possibly poisoned training set \cite{SS, AC, CI, Differential_Privacy, huang2022backdoor}.
They can not be deployed at the user end where the classifier is already trained.
Inference-stage defenses detect whether a test sample is embedded with a backdoor trigger \cite{STRIP, Februus, SentiNet}.
They require test samples with the actual backdoor trigger, which are unavailable for our detection problem.
Thus, we will not further discuss these methods.

\vspace{-0.05in}
\section{Threat Model}\label{sec:threat_model}
\vspace{-0.025in}

X2X backdoor attacks refer to a family of backdoor attacks with arbitrary numbers of source classes each assigned with an arbitrary target class.
It covers many popular attacks with different settings including  the ``all-to-one'' (A2O) attack \cite{Targeted}, ``X-to-one'' (X2O) attack \cite{Shen2021BackdoorSF} (a.k.a. a ``partial backdoor'' \cite{NC}), ``one-to-one'' (O2O) attack \cite{SS}, and ``all-to-all'' (A2A) attack \cite{BadNet}.
The complete taxonomy of X2X backdoor attacks is shown in Fig. \ref{fig:venn_map}.
Formally, for a classification task with sample space ${\mathcal X}$ and label space ${\mathcal Y}$, an X2X backdoor attack can be defined as the following:

\begin{definition}\label{def:target_specific_attack}
({\bf X2X Backdoor Attack}) {\it An X2X backdoor attack against a victim classifier $f:{\mathcal X}\rightarrow{\mathcal Y}$ is specified by a trigger embedding function $\delta:{\mathcal X}\rightarrow{\mathcal X}$ and a subset ${\mathcal A}\subset{\mathcal Y}\times{\mathcal Y}$ of backdoor class pairs, satisfying: {\bf (1)} $\forall a=(s, t)\in{\mathcal A}$, $s\neq t$,  {\bf (2)} if $|{\mathcal A}|>1$\footnote{For attacks with only one backdoor class pair, i.e. $|{\mathcal A}|=1$, our method is still effective empirically (see Sec. \ref{subsec:exp_adaptive}) due to a ``collateral damage'' phenomenon observed by \citeauthor{Post-TNNLS} (2020).}, for any $a_i=(s_i, t_i)\in{\mathcal A}$ and $a_j=(s_j, t_j)\in{\mathcal A}$, $s_i\neq s_j$ if $a_i\neq a_j$.
A (perfectly) successful X2X attack will: {\bf (a)} jointly minimize ${\mathbb E}_{P_{XY|a}}[l(Y, f(\delta(X)))]$ over both $\delta$ and $f$, $\forall a\in{\mathcal A}$, and {\bf (b)} jointly minimize ${\mathbb E}_{P_{XY|a}}[l(Y, f(X))]$ over $f$ for all class pairs $a=(s, t)$ with $s=t$ (i.e., high accuracy on clean samples), where $l:{\mathcal Y}\times{\mathcal Y}\rightarrow{\mathbb R}$ is the loss function of classifier $f$.}
\end{definition}

{\bf Notes:} In Def. \ref{def:target_specific_attack}, $P_{XY|a}$ is the joint distribution of (source class) sample $X\in{\mathcal X}$ and (target) label $Y\in{\mathcal Y}$ conditioned on class pair $a\in{\mathcal Y}\times{\mathcal Y}$.
In particular, for any $a=(s,t)$, the marginal distribution $P_{Y|a}$ is a singleton at $Y=t$, and $X$ only depends on $s$, i.e., $P_{XY|a}(x,y)=P_{X|s}(x)\cdot\mathds{1}[y=t]$ for any $x\in\mathcal{X}$ and $y\in\mathcal{Y}$ where $\mathds{1}[\cdot]$ is the indicator function.
Thus, goal (a) can be achieved {\it only if} condition (2) holds; otherwise, there will be at least two class pairs in ${\mathcal A}$ with conflict minimization objectives.
Moreover, although $l$ can be any legitimate loss function for classification, for simplicity, in this paper, we consider the {\it 0-1 loss} with $l(Y_1, Y_2)=0$ if $Y_1=Y_2$ and $l(Y_1, Y_2)=1$ otherwise.
Finally, we do not specify the form of $\delta$ here, since our UMD is applicable to a variety of trigger types -- (e.g.) (1) \textit{image perturbation trigger} embedded by $\delta(X)=[X+v]_{\rm c}$, where $v$ is a small perturbation and $[\cdot]_{\rm c}$ is a clipping function, and (2) a \textit{patch trigger} embedded by $\delta(X)=(1-m)\odot X + m\odot u$, where $u$ is a small image patch, $m$ is a binary mask, and $\odot$ represents element-wise multiplication.

By definition, X2X attacks are different from the N2N attacks proposed by \citeauthor{xue2022imperceptible} (2022).
The latter refers to backdoor attacks with multiple triggers, each associated with a unique target class, which can be viewed as the joint deployment of multiple A2O attacks \cite{xue2022O2NN2O}.
By contrast, X2X attacks use a single trigger, with the main focus on different configurations of the (source, target) class pairs.
In Sec. \ref{subsec:exp_adaptive}, we show that UMD (with trivial generalization) can  easily detect N2N attacks.

In practice, X2X attacks can be easily launched by poisoning the training set of $f$, with $\delta$ prescribed by the attacker.
For many choices of $\delta$ (even without optimization), both (a) and (b) in Def. \ref{def:target_specific_attack} can be achieved by only optimizing over $f$ during the training.
Thus, the attacker does not need access to the training process, which is required by many other advanced attacks.
Moreover, X2X attacks are not detectable by existing methods without supervision.
REDs mostly assume that the attack is A2O \cite{NC}. 
MCDs need to train shadow models for a variety of attack settings, which cannot effectively cover all possible backdoor class pair configurations for X2X attacks \cite{META}.
Thus, we propose UMD (introduced next) to close this gap.

\begin{figure}[t]
\centering
\includegraphics[width=.52\columnwidth]
{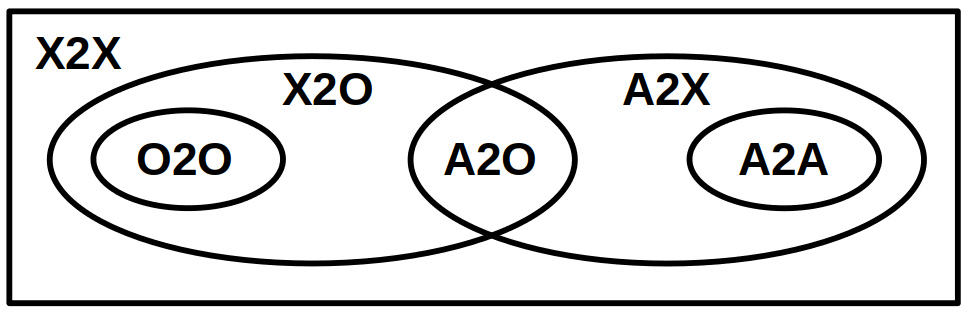}
\vspace{-0.1in}
\caption{Venn map for the family of X2X attacks, with `X' for ``arbitrary'', `A' for ``all'', and `O' for ``one''.}\label{fig:venn_map}
\vspace{-0.2in}
\end{figure}

\vspace{-0.05in}
\section{Method}\label{sec:method}
\vspace{-0.025in}

Next, we will first provide a formal problem statement for model detection against backdoor attacks, and then provide an overview of our proposed detection approach UMD, followed by a detailed introduction of UMD procedures.

{\bf Model detection problem}
For any potentially backdoor attacked classifier $f:{\mathcal X}\rightarrow{\mathcal Y}$, a defender aims to detect (without supervision) whether $f$ is backdoor attacked and infer all the backdoor class pairs (i.e.  set ${\mathcal A}$).
Similar to the importance of the target class inference for A2O attacks \cite{Competition}, for X2X attacks, the detected class pairs can be used to ``fix'' the classifier by ``unlearning'' the backdoor on these class pairs \cite{NC}.
The defender is assumed with the following {\it constraints}: (1) does not know \textit{a priori} if $f$ is attacked or not;
(2) no access to the training set or any samples embedded with the backdoor trigger;
(3) no access to any benign classifiers for reference (otherwise, one can use the benign classifier for the task);
(4) no prior knowledge about the number of backdoor class pairs or the assignment rules for the target classes.
Thus, the detection problem is \textit{unsupervised} due to the unavailability of both models with and without a backdoor.
Commonly, the defender is allowed to possess a small dataset ${\mathcal D}_{\rm c}$ containing clean samples for detection \cite{NC}.

{\bf Overview of UMD} To address the unavailability of the true backdoor trigger, UMD first reverse-engineers a putative trigger for each class pair using samples in ${\mathcal D}_{\rm c}$.
Different from prior works (e.g. \citeauthor{NC} (2019)) that assume an A2O attack and perform trigger reverse-engineering for each putative backdoor {\it target class}, our design makes class-pair-wise inference possible.
However, as a result, the premise behind those prior works -- the (image) trigger estimated for the backdoor target class will have a much smaller perturbation size than for all the other classes -- cannot be extended to our method with class-pair-wise trigger reverse-engineering.
Indeed, when there is an attack, the estimated trigger for the backdoor class pairs in ${\mathcal A}$ will have a small perturbation size due to the nature of the attacks.
But when there is no attack, the estimated trigger for some non-backdoor class pairs may also have a small perturbation size -- this is called an ``{\it intrinsic backdoor}'' (a.k.a. natural backdoor) \cite{PCBD, taog2022better}, which easily causes a false detection if the class-pair-wise perturbation size statistics are directly used for inference (as will be shown by our experiments in Sec. \ref{subsec:exp_ablation}).
To avoid such false detection, we propose a statistic ``{\it transferability}'' (TR), which is defined for each {\it ordered pair of class pairs} based on the reverse-engineered trigger (Sec. \ref{subsubsec:tr_def}).
We show that in ideal cases, TR from a backdoor class pair to another backdoor class pair is guaranteed to be no less than TR from a backdoor class pair to a non-backdoor class pair (Sec. \ref{subsubsec:tr_prop}).
Based on this property, UMD selects a subset of putative backdoor class pairs using the TR estimated for all ordered class pairs, by solving a proposed optimization problem (Sec. \ref{subsubsec:select_pair})
Then, an aggregation of the perturbation size statistics over all the selected class pairs is assessed by an {\it unsupervised}, bias-reduced anomaly detector (Sec. \ref{subsubsec:detection_inference}).
In summary, the set of class pairs being detected should have: (a) a large TR to any other class pair in the set and a small TR to any class pair not in the set, and (b) a small perturbation size for the reverse-engineered trigger.
The pipeline of our UMD is illustrated in Fig. \ref{fig:outline} and summarized by Alg. \ref{alg:graph_opt}.

\vspace{-0.05in}
\subsection{Transferability}\label{subsec:mt}
\vspace{-0.025in}

\subsubsection{Definition}\label{subsubsec:tr_def}
\vspace{-0.025in}

As motivated above, the TR statistic is defined for each ordered pair of class pairs based on the reverse-engineered trigger.
Since neither TR nor any part of the UMD pipeline is limited to any objective function or algorithm for trigger reverse-engineering, we define a {general} form for the trigger reverse-engineering problem as the following.
That is, for each $a=(s, t)\in{\mathcal Y}\times{\mathcal Y}$ $(s\neq t)$, we solve:
\begin{equation}\label{eq:re_main}
	\begin{aligned}
		& \underset{\delta}{\text{minimize}}
		& & {\mathbb E}_{P_{X|a}}[d(X, \delta(X))]\\
		& \text{s.t.}
		& & \delta \in \argmin_{\delta'} {\mathbb E}_{P_{XY|a}}[l(Y, f(\delta'(X)))]
	\end{aligned}
\end{equation}
Here $d:{\mathcal X}\times{\mathcal X}\rightarrow{\mathbb R}$ is a distance metric with respect to the trigger type, e.g. the $\ell_2$ norm $d(X, \delta(X))=||X-\delta(X)||_2$ for image perturbation triggers \cite{Post-TNNLS}.
The distance is minimized since image triggers are typically designed to be human-imperceptible.
Moreover, if $a$ is a backdoor class pair, the set of $\delta$ satisfying the constraint of \eqref{eq:re_main} will include the true backdoor trigger due to the goal (a) of the attacker in Def. \ref{def:target_specific_attack}.
Empirically, for each $a=(s, t)$, problem \eqref{eq:re_main} can be solved on clean samples in ${\mathcal D}_{\rm c}$ from class $s$ \cite{NC}.
Denoting the reverse-engineered trigger (i.e. the optimal solution to \eqref{eq:re_main}) for each class pair $a$ by $\delta_a$, we define the TR statistic as the following:
\begin{definition}\label{def:tr}
({\bf Transferability (TR)}) {\it For any class pair $a_i=(s_i, t_i)$, $s_i\neq t_i$, with a reverse-engineered trigger $\delta_{a_i}$, and 0-1 loss $l(\cdot, \cdot)$, TR from $a_i$ to any other class pair $a_j=(s_j, t_j)$, $a_j\neq a_i$ and $s_j\neq t_j$, is defined by:
\begin{equation}\label{eq:mt_definition}
    T_{a_ia_j} \triangleq\, 1 - {\mathbb E}_{P_{XY|a_j}}[l(Y, f(\delta_{a_i}(X)))].
\end{equation}}
\end{definition}
Based on the notes below Def. \ref{def:target_specific_attack}, the expectation in Eq. \eqref{eq:mt_definition} is equivalent to ${\mathbb E}_{P_{X|s_j}}[l(t_j, f(\delta_{a_i}(X)))]$.
Thus, empirically, $T_{a_ia_j}$ can be estimated using the clean samples from class $s_j$ in ${\mathcal D}_{\rm c}$ and the trigger reverse-engineered for class pair $a_i$.
The form of Eq. \eqref{eq:mt_definition} is chosen for $l$ being 0-1 loss with the value of TR scaled to $[0, 1]$ for simplicity, though other forms can be adopted for different choices of the loss function.
In plain language, $T_{a_ia_j}$ represents the misclassification rate to class $t_j$ when the trigger $\delta_{a_i}$ (estimated for class pair $a_i$) is applied to examples from class $s_j$.

\vspace{-0.025in}
\subsubsection{Property}\label{subsubsec:tr_prop}
\vspace{-0.025in}

Next, we show that TR is intrinsically suitable for identifying backdoor class pairs.
Consider an {\it arbitrary} set of class pairs ${\mathcal A}'=\{a_1, \cdots, a_k\}$ satisfying both conditions (1) and (2) in Def. \ref{def:target_specific_attack}, and with $P_A(a)>0$ for $\forall a\in{\mathcal A}'$.
For any trigger embedding function $\delta$, we denote $X^{\delta}\triangleq\delta(X)$ as the random variable for samples with a trigger embedded by $\delta$.
Then, the set of Bayes classifiers \cite{bayes} for (optimal) estimation of $Y$ from $X^{\delta}$ can be written as:
\begin{equation}\label{eq:bayes_classifier}
{\mathcal F}^{\delta} = \{f\in{\mathcal F} | \mathbb{E}_{P_{X^{\delta}Y}} [l(Y, f(X^{\delta}))]=R^{\mathcal F}(Y|X^{\delta})\}
\end{equation}
where $P_{X^{\delta}Y}=P_{XY}\cdot P_{X^{\delta}|X}$ is the joint distribution of $X^{\delta}$ and $Y$, $l(\cdot, \cdot)$ is the classification loss (i.e. 0-1 loss here), ${\mathcal F}$ is the set of all legitimate classifiers\footnote{For example, all classifiers with the same architecture as the one to be inspected but with different parameter values.}, and
\begin{equation}\label{eq:bayes_risk}
	R^{\mathcal F}(Y|X^{\delta}) = \min_{f\in{\mathcal F}} \mathbb{E}_{P_{X^{\delta}Y}} [l(Y, f(X^{\delta}))]
\end{equation}
is the Bayes risk over all classifiers in ${\mathcal F}$ for estimating $Y$ from $X^{\delta}$.
Here, we assume the minimum always exists for simplicity.
Similarly, for each class pair $a\in{\mathcal A}'$, we denote the set of ``{\it class-pair-conditional}'' Bayes classifiers as ${\mathcal F}_a^{\delta}$ and the associated Bayes risk as $R^{\mathcal F}_a(Y|X^{\delta})$, by replacing $P_{X^{\delta}Y}$ in both Eq. \eqref{eq:bayes_classifier} and \eqref{eq:bayes_risk} with $P_{X^{\delta}Y|a}=P_{XY|a}\cdot P_{X^{\delta}|X}$.
These classifiers in ${\mathcal F}_a^{\delta}$ are optimal for predicting $Y$ from $X^{\delta}$, with $X^{\delta}$ and $Y$ both conditioned on the class pair $a$. Then, we have the following theorem for the transferability of reverse-engineered triggers:

\begin{theorem}\label{thm:main}
({\bf Optimal Transferability Condition}) {\it For any class pair $a\in{\mathcal A}'$, consider a trigger embedding function $\delta$ that minimizes $R^{\mathcal F}_a(Y|X^{\delta})$. Then, $\delta$ minimizes:
\begin{equation}\label{eq:main_lhs}
    \min_{f\in{\mathcal F}_a^{\delta}} \sum_{a'\in{\mathcal A}'\setminus a} P_{A|A\neq a}(a') \mathbb{E}_{P_{X^{\delta}Y|a'}} [l(Y, f(X^{\delta}))]
\end{equation}
if and only if $\delta$ also minimizes $R^{\mathcal F} (Y|X^{\delta})$.}
\end{theorem}

\begin{proof}[Proof (sketch)]
First, we derive the lower bound of \eqref{eq:main_lhs} over $\delta$. Then, sufficiency is proved by showing that the lower bound will be achieved if $\delta$ minimizes $R^{\mathcal F} (Y|X^{\delta})$, while necessity is proved by showing that the lower bound cannot be achieved if $\delta$ does not minimize $R^{\mathcal F} (Y|X^{\delta})$ via contradiction. The complete analysis is shown in Apdx. \ref{sec:analysis}.
\end{proof}

{\bf Remarks:} For any class pair $a$ and classifier $f$, the reverse-engineered trigger satisfying the constraint of problem \eqref{eq:re_main} should also minimize $R^{\mathcal F}_a(Y|X^{\delta})$ if $f$ is a Bayes classifier conditioned on $a$.
Thus, based on goal (a) in Def. \ref{def:target_specific_attack}, $\delta$ considered by the theorem may be a trigger reverse-engineered for some backdoor class pair $a$ of a successful X2X attack.
In this case, based on Def. \ref{def:tr}, the conditional expectation in \eqref{eq:main_lhs} for each $a'\in{\mathcal A}'\setminus a$ represents one minus the TR statistic from $a$ to $a'$.
Thus, the theorem shows the condition for $\delta$ maximizing the expected TR from $a$ to all the other class pairs in ${\mathcal A}'$, which is that $\delta$ also minimizes $R^{\mathcal F} (Y|X^{\delta})$ -- the Bayes risk without any class-pair-conditioning.
Apparently, this optimal transfer condition holds if ${\mathcal A}'\subset{\mathcal A}$ contains only backdoor class pairs of a successful X2X attack, with $f$ being a Bayes classifier on ${\mathcal A}'$ and $\delta$ being the actual backdoor trigger.
Thus, for a perfectly successful attack and optimal trigger reverse-engineering, if we apply Thm. \ref{thm:main} to any set ${\mathcal A}'$ of two class pairs with at least one being a backdoor class pair, we will have the guarantee that {\it TR from a backdoor class pair to another backdoor class pair is no less than TR from a backdoor class pair to a non-backdoor class pair}.
Empirically, we will likely observe large TRs (possibly close to 1) for any ordered pair of class pairs in ${\mathcal A}'$ if the set is pure in backdoor class pairs.
Otherwise, there will likely be at least two class pairs in ${\mathcal A}'$ with a small TR from either direction.

\begin{algorithm}[t]
\small
\caption{UMD against X2X backdoor attacks}
\begin{algorithmic}[1]\label{alg:graph_opt}

\STATE {\bf Input:} a classifier $f$; a small, clean dataset ${\mathcal D}_{\rm c}$, a desired significance level $\beta$ for anomaly detection.

\STATE {\bf Compute TR statistics:}

\STATE Get $\delta_a$ by solving \eqref{eq:re_main} on ${\mathcal D}_{\rm c}$ for $\forall a=(s, t)\in{\mathcal Y}\times{\mathcal Y}\setminus{\mathcal B}$.

\STATE Compute $T_{a_ia_j}$ by Eq. \eqref{eq:mt_definition} on ${\mathcal D}_{\rm c}$ for $\forall a_i\in{\mathcal Y}\times{\mathcal Y}\setminus{\mathcal B}$ and $\forall a_j=(s_j, t_j)\in{\mathcal Y}\times{\mathcal Y}\setminus{\mathcal B}$, $a_j\neq a_i$.

\STATE {\bf Select a set $\hat{\mathcal A}$ of putative backdoor class pairs:}

\STATE Initialize $\hat{\mathcal A}_2 = \argmax_{\{a_i=(s_i,t_i), a_j=(s_j,t_j)\}, s_i\neq s_j} T_{a_ia_j}$.

\FOR{$n=3:|{\mathcal Y}|$}

\STATE $\hat{\mathcal A}_{n-1}^{\rm c}=\{a=(s, t)\notin\hat{\mathcal A}_{n-1} | s\neq s' \,\, \text{for} \,\, \forall a'=(s', t')\in\hat{\mathcal A}_{n-1}\}$.

\STATE $a^{\ast}=\argmax_{a\in\hat{\mathcal A}_{n-1}^{\rm c}} H(\hat{\mathcal A}_{n-1} \cup a)$.

\STATE $\hat{\mathcal A}_n = \hat{\mathcal A}_{n-1} \cup a^{\ast}$.

\ENDFOR

\STATE $n^{\ast} = \argmax_{n\in\{3, \cdots, |{\mathcal Y}|\}} H(\hat{\mathcal A}_n)$

\STATE $\hat{\mathcal A} = \hat{\mathcal A}_{n^{\ast}}$

\STATE {\bf Unsupervised anomaly detection:}

\STATE Compute $r$ on ${\mathcal D}_{\rm c}$ using all $\{\delta_a\}$ and $\hat{\mathcal A}$ by Eq. \eqref{eq:mad} and \eqref{eq:anomaly_score}.

\STATE Compute $\theta(\beta, N)$ for $N=|{\mathcal Y}\times{\mathcal Y}\setminus{\mathcal B}|-|\hat{\mathcal A}|$ by Eq. \eqref{eq:threshold}.

\STATE {\bf Output:} If $r>\theta(\beta, N)$, there is an attack with backdoor class pairs $\hat{\mathcal A}$; otherwise, there is no attack.

\end{algorithmic}
\end{algorithm}

\vspace{-0.05in}
\subsection{Detection Inference}\label{subsec:inference}
\vspace{-0.025in}

\subsubsection{Select Putative Backdoor Class Pairs}\label{subsubsec:select_pair}
\vspace{-0.025in}

Due to the absence of supervision, it is hard to choose a threshold on TR to identify the backdoor class pairs directly if there is any.
Moreover, a naive combination of TR with other statistics such as the perturbation size of the reverse-engineered trigger cannot effectively detect backdoor class pairs, while still causing a high false detection rate (as will be shown by our experiments in Sec. \ref{subsec:exp_ablation}).
Thus, we propose to use TR to select a set $\hat{\mathcal A}$ of putative backdoor class pairs for further inference.
Based on our analysis for TR, if there is an attack, we expect: {\bf (1)} a large TR for any ordered pair of class pairs in $\hat{\mathcal A}$, {\bf (2)} a small TR from any class pair in $\hat{\mathcal A}$ to class pairs outside $\hat{\mathcal A}$, {\bf (3)} $\hat{\mathcal A}$ satisfies the conditions in Def. \ref{def:target_specific_attack} for valid X2X attacks.
Accordingly, we propose to solve the following optimization problem:
\begin{equation}\label{eq:clustering}
\begin{aligned}
    & \underset{\hat{\mathcal A}\subset{\mathcal Y}\times{\mathcal Y}\setminus{\mathcal B}}{\text{maximize}}
    & & H(\hat{\mathcal A}) = \min_{a\in\hat{\mathcal A}} \frac{\sum_{a'\in\hat{\mathcal A}\setminus a} (T_{aa'} + T_{a'a})}{2(|\hat{\mathcal A}| - 1)}\\
    & & &- \max_{a\notin\hat{\mathcal A}} \frac{\sum_{a'\in\hat{\mathcal A}} T_{a'a}}{|\hat{\mathcal A}|}\\
    & \text{subject to}
    & & s \neq s', \forall a=(s, t), a'=(s', t') \in \hat{\mathcal A}
\end{aligned}
\end{equation}
where ${\mathcal B}=\{(s, t)\in{\mathcal Y}\times{\mathcal Y}|s=t\}$ is the set of all ``identical'' pairs.
Clearly, for problem \eqref{eq:clustering}, the two terms in the objective function and the constraint are designed to satisfy the requirements (1)-(3), respectively.
In particular, the second term of $H(\hat{\mathcal A})$ is critical in practice when the actual number of backdoor class pairs is unknown.
Without this term, we will likely obtain a parsimonious set $\hat{\mathcal A}$ of two class pairs with the top ``mutual-TR''.
Finally, we propose to solve problem \eqref{eq:clustering} using an agglomerative algorithm {\it without any hyperparameter}, as detailed by lines 5-13 of Alg. \ref{alg:graph_opt}.

\subsubsection{Unsupervised Anomaly Detection}\label{subsubsec:detection_inference}

Since $\hat{\mathcal A}$ will always be selected regardless of the presence of attack, we still need to infer whether $\hat{\mathcal A}$ is indeed a set of backdoor class pairs.
Inspired by previous works, we design an anomaly detector based on median absolute deviation (MAD) \cite{MAD}.
The anomaly detector uses the trigger perturbation/patch size $z_a\triangleq{\mathbb E}_{P_{X|a}}[d(X, \delta_a(X))]$ empirically estimated for each class pair $a=(s, t)$ on the clean samples ${\mathcal D}_{\rm c}$ as the detection statistic.
Under the null hypothesis of ``no attack'', all detection statistics are associated with non-backdoor class pairs and follow some null distribution characterized by the median statistic and MAD.
Different from prior works, our estimation of MAD (denoted by $\sigma$ below) is performed on $\forall a\notin\hat{\mathcal A}$ which are likely non-backdoor class pairs, i.e.:
\begin{equation}\label{eq:mad}
\sigma = {\rm med}_{a\notin\hat{\mathcal A}}(|z_a^{-1} - {\rm med}_{a'\notin\hat{\mathcal A}} z^{-1}_{a'}|)
\end{equation}
where ${\rm med}$ represents median. 
The reciprocal is taken such that the outlier statistics corresponding to small trigger sizes, if there are any, will stay at the tail of the null distribution.
Compared with other detectors that use all statistics to estimate MAD (since they do not select putative backdoor class pairs like us), our estimation will not suffer from the bias caused by the possible involvement of backdoor statistics.
Then, we assess the atypicality of $z_a$ for $\forall a\in\hat{\mathcal A}$ through aggregation using an anomaly score computed by:
\begin{equation}\label{eq:anomaly_score}
r = ({\rm med}_{a\in\hat{\mathcal A}} z_{a}^{-1} - {\rm med}_{a'\notin\hat{\mathcal A}} z_{a'}^{-1}) / (1.4826 \cdot \sigma)
\end{equation}
where the constant 1.4826 is a scaling factor such that the scaled MAD can be viewed as an analog to the standard deviation of the null distribution under Gaussian assumption \cite{rousseeuw-qn-1993}.
The aggregation, i.e. the median of $z_a^{-1}$ for $\forall a\in\hat{\mathcal A}$, helps to avoid false detection caused by any $a\in\hat{\mathcal A}$ with an outlier statistic (e.g. for an intrinsic backdoor) when there is actually no attack.
In summary, the anomaly score $r$ describes how many ``standard deviations'' the aggregated statistic is away from the median.

To test whether $r$ is an outlier to the underlying null distribution, we propose a method to determine a confidence threshold in adaption to the number of ``null statistics'', i.e. $N=|{\mathcal Y}\times{\mathcal Y}\setminus{\mathcal B}|-|\hat{\mathcal A}|$, which is largely dependent on the number of classes $|{\mathcal Y}|$.
Let $R_1, \cdots, R_N$ be i.i.d. random variables following some null density form, e.g., a standard Gaussian distribution in here.
It is easy to show that for any given $\Theta$, ${\rm Prob}(\max_{i=1,\cdots, N} R_i > \Theta)\rightarrow 1$ as $N\rightarrow\infty$.
In other words, with a constant threshold, a false detection will be easily made when $N$ is large.
Thus, we obtain a threshold $\theta(\beta, N)$ based on both a prescribed confidence level $1-\beta$ (e.g. $\beta=0.05$ by convention) and $N$ by solving $\theta$ from ${\rm Prob}(\max_{i=1,\cdots, N} R_i > \theta)\leq\beta$, which gives:
\begin{equation}\label{eq:threshold}
\theta(\beta, N) = \Phi^{-1} ((1-\beta)^{1/N})
\end{equation}
where $\Phi^{-1}$ is the inverse of the standard Gaussian CDF.
Then, if $r>\theta(\beta, N)$, we claim with confidence $1-\beta$ (a.k.a. $\beta$-significance) that the classifier is attacked with backdoor class pairs $\hat{\mathcal A}$; otherwise, no backdoor attack.

\vspace{-0.05in}
\section{Experiment}\label{sec:exp}
\vspace{-0.025in}

First, we show that our unsupervised UMD outperforms five SOTA baselines (even with supervision) by at least 17\%, 4\%, and 8\% on CIFAR-10, GTSRB, and Imagenette, respectively, in the average model inference accuracy against various X2X attacks.
Second, in our ablation study on CIFAR-10, we justify our design choices for UMD.
Third, we show that UMD can even detect X2X attacks with two advanced triggers and address four different types of adaptive attacks.
Finally, we show that the class pairs detected by UMD can be used to ``fix'' the backdoored model.

\vspace{-0.025in}
\subsection{Setup}\label{subsec:exp_main_setup}
\vspace{-0.025in}

\begin{table}[t!]
\vspace{-0.1in}
\setlength{\tabcolsep}{3.5pt}
\caption{Designed functionalities and detection capabilities of UMD compared with five SOTA baselines. UMD is the only unsupervised method against X2X attacks with pair inference. Empirically, UMD can also detect O2O attacks as shown in Tab. \ref{tab:detection_advanced}.}
\centering{
    \scalebox{0.8}{%
        \begin{tabular}{c|cccccc}
            \toprule \hline
            & NC & ABS & PT-RED & MNTD & K-Arm & UMD (ours)\\ \hline
            A2O & \checkmark & \checkmark & \checkmark & \checkmark & \checkmark & \checkmark\\ \hline
            O2O &  &  & \checkmark &  & \checkmark & $\triangle$ \\ \hline
            X2O &  &  & \checkmark &  & \checkmark & \checkmark\\ \hline
            A2Ar &  &  &  &  & & \checkmark\\ \hline
            A2X &  &  &  &  & & \checkmark\\ \hline
            X2X &  &  &  &  & & \checkmark\\\hline
            detect pairs &  &  &  &  & & \checkmark\\\hline
            unsupervised & \checkmark &  & \checkmark &  & & \checkmark\\ \hline\bottomrule
        \end{tabular}
    }
}
\label{tab:detection_capability}
\vspace{-0.15in}
\end{table}

{\bf Dataset:}
We consider three benchmark image datasets, CIFAR-10 \cite{CIFAR10}, GTSRB \cite{GTSRB}, and Imagenette \cite{ImageNet}, which contain color images (with resolution $32\times32$, $32\times32$ (resized), and $224\times224$, respectively) with 10, 43, and 10 classes, respectively.
In our experiments, we follow the standard train-test split for each dataset (see Apdx. \ref{subsec:exp_main_supp_dataset} for details).\\
{\bf Backdoor trigger:}
We consider two common triggers: 1) a large, perturbation-based trigger with a big `X' shape, and 2) a local patch trigger with a random color and a random location for each attack.
Examples of these triggers are shown in Fig. \ref{fig:trigger_example}, with more details in Apdx. \ref{subsec:exp_main_supp_trigger}.\\
{\bf Attack setting:}
We first consider the classical A2O attack addressed by most existing works for all three datasets.
The target class for each A2O attack is randomly selected.
Then we consider a general all-to-all (A2Ar) attack with a random bijection mapping between the source and target classes.
Note that the classical A2A attack by \citeauthor{BadNet} (2017) uses rotational target assignment and is a special case of the A2Ar attack considered here.
For each dataset, we also consider several X2X attack settings other than A2O and A2Ar.
On CIFAR-10, we consider 2to2, 5to5, and 8to8 attacks;
on GTSRB, we consider 20to20, 30to30, and 40to40 attacks;
on Imagenette, we consider 3to3, 5to5, and 8to8 attacks.
The backdoor class pairs for each X2X attack are randomly selected.
Moreover, for each attack on CIFAR-10, GTSRB, and Imagenette, we create 300, 70, and 200 poisoning instances per source class, respectively.\\
{\bf Training:}
For each attack setting on each dataset, we train 10 classifiers under attack with each of the two triggers respectively.
For the 8to8 and the A2Ar settings on Imagenette, the attacks with the patch trigger are mostly unsuccessful; thus, they are excluded from our experiments.
In total, our main evaluation of the detection performance involves ($(5\times3\times2 - 2\times1\times1)\times10=$) 280 classifiers being attacked.
For model architecture, we use ResNet-18 \cite{ResNet} for CIFAR-10 and Imagenette, and the winning model on the leaderboard \cite{GTSRB_Leaderboard} for GTSRB.
Detailed training configurations are shown in Apdx \ref{subsec:exp_main_supp_training}.
All the attacks we created are successful with attack success rates (ASRs) $>78\%$ and negligible degradation in clean test accuracy (ACC) (see Tab. \ref{tab:asr_acc} in Apdx. \ref{subsec:exp_main_supp_training}).\\
{\bf Evaluation metric:}
We define a \underline{m}odel \underline{i}nference \underline{a}ccuracy ({\bf MIA}) as the proportion of {\it correct inference} for a group of classifiers.
MIA is equivalent to the true positive rate (or one minus the false positive rate) if all classifiers in the group are attacked (or benign).
For each {\it true positive} model inference by UMD, we also define a \underline{p}air \underline{d}etection \underline{r}ate ({\bf PDR}) which is the proportion of backdoor class pairs being successfully detected.
Note that the false positive rate for pair inference (by incorrectly recognizing a non-backdoor class pair as a backdoor class pair) will always be small since UMD detects at most $K$ (out of $K(K-1)$) class pairs, where $K$ is the number of classes.
Thus, we neglect it for brevity.\\
{\bf Baselines:}
We compare our UMD with the following SOTA baselines, including Neural Cleanse (NC) \cite{NC}, ABS \cite{ABS}, PT-RED \cite{Post-TNNLS}, MNTD \cite{META}, and K-Arm \cite{Shen2021BackdoorSF}.
For a fair comparison, we set the confidence level for model inference to 95\% (i.e. 5\% {\it desired} false positive rate) for NC and PT-RED equipped with unsupervised threshold selection.
For ABS, MNTD, and K-Arm which require supervision to select the detection threshold, we set the overall {\it actual} false positive rate (for all datasets and settings) to 5\% while \textit{maximizing} their true positive rates for model inference.
The designed functionalities and detection capabilities of these methods are shown in Tab. \ref{tab:detection_capability}, compared with UMD. More details about these methods are shown in Apdx. \ref{subsec:exp_main_supp_review}.\\
\textbf{Experimental Details:}
For our UMD, we consider the trigger reverse-engineering algorithms used by PT-RED and NC, respectively, to cover both the perturbation trigger and the patch trigger.
That is, we execute Alg. \ref{alg:graph_opt} with both algorithms, and a classifier is deemed to be attacked if any of the two executions claim a detection.
In particular, PT-RED assumes that the trigger is an additive image perturbation incorporated by $\delta(x)=[x+v]_{\rm c}$ with a small $||v||_2$, where $[\cdot]_{\rm c}$ is a clipping function \cite{Post-TNNLS}.
Its reverse engineer algorithm is similar to the way to generate a universal adversarial perturbation \cite{DeepFool_Univ} -- for any class pair $(s, t)$, a perturbation $v$ is initialized to zero and updated using gradient-based approaches, until a high misclassification fraction from class $s$ to class $t$ is achieved. 
NC assumes a patch trigger $u$ embedded by $\delta(x)=(1-m)\odot x + m\odot u$ using a binary mask $m$ with a small patch size $||m||_1$, where $\odot$ represents element-wise multiplication \cite{NC}.
The reverse engineering algorithm of NC also solves an optimization problem for each class pair $(s, t)$ to achieve a high misclassification fraction from class $s$ to class $t$ while minimizing the patch size $||m||_1$.
For all three datasets, the two algorithms consume merely 10 and 20 trigger-free images (correctly predicted by the classifier to be inspected) per class, respectively.
More details about these two algorithms can be found in Apdx. \ref{subsec:exp_main_supp_review_re}.
Again, our UMD is not limited to any particular algorithms for trigger reverse-engineering, allowing the potential incorporation with more recent or even future techniques \cite{wang2023unicorn}.
For the selection of candidate backdoor class pairs, we repeat lines 6-13 of Alg. \ref{alg:graph_opt} five times, each with a different initialization, and pick the best optimal solution to avoid poor local optimum.
For the anomaly detection step, we use the same confidence threshold of 95\% (i.e. $\beta=0.05$) as the other detectors for a fair comparison.
Results for other confidence levels are shown in Apdx. \ref{subsec:exp_main_supp_confidence}.

\begin{table}[t!]
\vspace{-0.1in}
\setlength{\tabcolsep}{4pt}
\caption{MIA of UMD for various X2X attacks and benign classifiers on CIFAR-10, GTSRB, and Imagenette, compared with five SOTA detectors.
MIAs of ABS, MNTD, and K-Arm on benign classifiers are manually fixed to control the false positive rates; thus are ``not applicable'' (n.a.).
UMD outperforms the five SOTA detectors (some even with supervision) on all three datasets by a clear margin in the average MIA over the X2X attacks.}
\centering{
\scalebox{0.83}{%
\begin{tabular}{c|c|cccccc}
    \multicolumn{8} {c} {(a) CIFAR-10}\\
\toprule \hline
    Setting & Benign & A2O & 2to2 & 5to5 & 8to8 & A2Ar & Avg \\
    \hline
    NC &0.60 &0.55 &0.20 &0.20 &0.30 & 0.30&0.31 \\
    \hline
    ABS & n.a. &0.90 &0.40 &0.15 &0.20 &0.20 &0.37 \\
    \hline
    PT-RED &0.70 &0.55 &0.40 &0.35 &0.30 &0.45 &0.41 \\
    \hline
    MNTD & n.a. &0.45 &0.65 &0.40 &0.25 & 0&0.35 \\
    \hline
    K-Arm & n.a. &1.0 &0.90 &0.70 & 0.65&0.45 &0.74 \\
    \hline
    \textbf{UMD} & 0.90 & 0.90 & 0.90 & 0.95 & 0.85 & 0.95 & \textbf{0.91} \\
    \hline \bottomrule
    \multicolumn{8} {c} {(b) GTSRB}\\
\toprule \hline
    Setting & Benign & A2O & 20to20 & 30to30 & 40to40 & A2Ar & Avg \\
    \hline
    NC &0.90 &0.85 &0.30 &0.25 &0.35 &0.35 &0.42 \\
    \hline
    ABS & n.a. &0.35 &0.25 &0.10 &0.20 &0.10 &0.20 \\
    \hline
    PT-RED &0.20 &0.65 &0.50 &0.30 &0.55 & 0.55&0.51 \\
    \hline
    MNTD & n.a. &0.25 &0.15 &0.15 &0.15 &0 &0.14 \\
    \hline
    K-Arm & n.a. &1.0 &0.95 &0.85 &0.80 &0.75 &0.87 \\
    \hline
    \textbf{UMD} & 0.90 & 0.95 & 0.80 & 0.90 & 0.90 & 1.0 & \textbf{0.91} \\
    \hline \bottomrule
    \multicolumn{8} {c} {(c) ImageNette}\\
\toprule \hline
    Setting & Benign & A2O & 3to3 & 5to5 & 8to8 & A2Ar & Avg \\
    \hline
    NC &0.90 &0.85 &0.30 &0.15 &0.05 &0.15 &0.30 \\
    \hline
    ABS & n.a. &1.0 &0.80 &0.40 &0.70 &0.70 &0.72 \\
    \hline
    PT-RED &0.80 &0.60 &0.45 &0.20 &0.10 &0 &0.27 \\
    \hline
    MNTD & n.a. &0.55 &0.50 &0.50 &0.30 &0.40 &0.45 \\
    \hline
    K-Arm & n.a. &0.90 & 0.60 &0.65 & 0.90 &0.80 &0.77 \\
    \hline
    \textbf{UMD} & 0.80 & 0.90 & 0.75 & 0.80 & 0.80 & 1.0 & \textbf{0.85} \\
    \hline \bottomrule
\end{tabular}
}
}
\label{tab:detection}
\vspace{-0.175in}
\end{table}

\vspace{-0.025in}
\subsection{Detection Performance}\label{subsubsec:exp_main_results}
\vspace{-0.025in}

As shown in Tab. \ref{tab:detection}, UMD clearly outperforms the five SOTA baselines on all three datasets in terms of the average MIA over the X2X attacks on each dataset.
In particular, most of these SOTA baselines exhibit some detection capability against A2O attacks they are designed for but fail against X2X attacks with more than one target class.
In contrast, UMD performs uniformly well against all X2X attacks, with even better control of the false positive rate (reflected by the generally higher MIA on benign classifiers) compared with the other two unsupervised detectors, NC and PT-RED.
We note that among the five SOTA baselines, K-Arm achieves the best average MIA against X2X attacks for all three datasets.
A possible reason is that K-Arm can effectively reverse-engineer the trigger for O2O attacks, while all X2X attacks can be viewed as a joint deployment of multiple O2O attacks sharing the same trigger.
However, K-Arm requires supervision to determine if a reverse-engineered trigger is associated with the backdoor, which is infeasible for practical backdoor detection problems.
But even with the supervision to maximize its performance, K-Arm is still outperformed by our \textit{unsupervised} UMD by 17\%, 4\%, and 8\% on CIFAR-10, GTSRB, and Imagenette, respectively, in terms of the average MIA over the X2X attacks for each dataset.
Finally, we show the pair inference performance of UMD in Tab. \ref{tab:detection_pair} since the other methods are not designed with such functionality.
UMD achieves high average PDRs for most X2X settings on the three datasets.
The relatively low PDRs, e.g. for A2Ar attacks on Imagenette, are likely due to the existence of intrinsic backdoor class pairs.

\begin{table}[t!]
\vspace{-0.15in}
    \setlength{\tabcolsep}{5pt}
    \caption{Average PDR of UMD over successfully detected attacks for the three datasets.}
    \centering{
        \scalebox{0.83}{%
            \begin{tabular}{c|c|ccccc}
                \toprule \hline
                \multirow{2}{*}{CIFAR-10} & Setting & A2O & 2to2 & 5to5 & 8to8 & A2Ar\\
                & Avg PDR & 0.93 &0.92 &1.0 &0.88 & 0.98 \\ \hline
                \multirow{2}{*}{GTSRB} & Setting & A2O & 20to20 & 30to30 & 40to40 & A2Ar\\
                & Avg PDR & 0.90 &0.72  &0.83 &0.79 & 0.86 \\ \hline
                \multirow{2}{*}{Imagenette} & Setting & A2O & 3to3 & 5to5 & 8to8 & A2Ar\\
                & Avg PDR & 0.96 &0.87 &0.75 &0.70 &0.65 \\ 
                \hline
                \bottomrule
            \end{tabular}
        }
    }
\label{tab:detection_pair}
\vspace{-0.2in}
\end{table}

\vspace{-0.025in}
\subsection{Ablation Study}\label{subsec:exp_ablation}
\vspace{-0.025in}

First, we show the advantages of using the proposed TR statistic and the associated clustering approach for backdoor detection by comparing UMD with its two baseline variants.
The first variant $\text{UMD}^{\dag}$ directly applies a MAD-based anomaly detector to triggers reverse-engineered for all class pairs, {\it without} using the TR statistic.
The second variant $\text{UMD}^{\ddag}$ uses TR simply as a secondary statistic {\it without} our clustering technique.
More details about these two baseline variants are shown in Apdx. \ref{subsec:exp_ablation_variants}.
For a demonstration, we consider the 2to2, 5to5, 8to8, and A2Ar attacks on CIFAR-10 with the perturbation trigger (i.e. 10 backdoored classifiers per setting).
We also use the 10 benign classifiers on CIFAR-10 to evaluate the false detection rate.

As shown in Tab. \ref{tab:detection_ablation}, though the desired false positive rate is set to 5\%, the actual ones for the two baseline variants are very high (reflected by the low MIAs on the benign classifiers).
Such high false positive rates cannot be alleviated even with alternative confidence levels, as shown in Apdx. \ref{subsec:exp_ablation_variants_auc}.
In contrast, UMD achieves a 93\% overall MIA as averaged over both attacked and benign classifiers with equal weights, showing a strong detection capability against X2X attacks with a controlled false detection rate.
Moreover, UMD achieves good performance in class pair inference, which is generally better than the two baseline variants.

\begin{table}[t!]
\vspace{-0.15in}
\setlength{\tabcolsep}{1.5pt}
\caption{MIA and average PDR of UMD, compared with the two baseline variants of UMD, against 2to2, 5to5, 8to8, and A2Ar attacks (with the perturbation trigger) on CIFAR-10.
Both variants of UMD favor predicting an ``attack'', resulting in low MIAs on benign classifiers (i.e. high false positive rates).
UMD achieves the best overall MIA (computed by adding the benign MIA with the average MIA for all attacks and then dividing by two).}
\centering{
    \scalebox{0.83}{%
        \begin{tabular}{c|cccccccccc}
            \toprule \hline
            \multirow{2}*{} & \multicolumn{2}{c}{2to2} & \multicolumn{2}{c}{5to5} & \multicolumn{2}{c}{8to8} & \multicolumn{2}{c}{A2Ar} & Benign & \multirow{2}*{\shortstack{Overall\\MIA}}\\ 
            \cline{2-10}
            & MIA & PDR & MIA & PDR & MIA & PDR & MIA & PDR & MIA \\ \hline
            $\text{UMD}^{\dag}$ & 1.0 & 0.90 & 1.0 & 0.94 &1.0 & 0.84 &1.0 & 0.72 & \cellcolor{red!25}0 & 0.50 \\ \hline
                $\text{UMD}^{\ddag}$ & 1.0 & 0.45 & 1.0 & 0.82 & 1.0 & 0.73 & 1.0 & 0.53 & \cellcolor{red!10}0.40 & 0.70 \\ \hline
            \textbf{UMD} & 1.0 & 0.85 & 1.0 & 1.0 & 0.90 & 0.83 & 0.90 & 0.92 & 0.90 & {\bf 0.93}\\
            \hline\bottomrule
        \end{tabular}
    }
}
\label{tab:detection_ablation}
\vspace{-0.2in}
\end{table}

Next, we show the influence of the hyperparameters on UMD.
Since UMD does not involve any tunable hyperparameters in the inference step, we study the influence of the hyperparameters used by the trigger reverse-engineering algorithms on our UMD.
In particular, we focus on the number of images and the targeted misclassification fraction used by \citeauthor{Post-TNNLS} (2020) for trigger reverse-engineering.
Note that for X2X attacks, the ASR for a backdoor class pair is typically less than 100\%.
Thus, in principle, the defender should avoid using an overly large targeted misclassification fraction; otherwise, trigger reverse-engineering may fail to produce an accurate estimation of the actual backdoor trigger.
As shown in Fig. \ref{fig:ablation_ni_pi}, UMD performs uniformly well for targeted misclassification fractions less than 1, giving a large freedom to choose this hyperparameter.

As for the number of images, UMD prefers even fewer (though $>1$) images for trigger reverse-engineering than the default setting by \citeauthor{Post-TNNLS} (2020).
Note that triggers reverse-engineered on a large number of images may easily contain class-discriminate features that transfer well between non-backdoor class pairs (especially those sharing the same target class) and lead to a wrong detection.
In practice, the suitable number of images for trigger reverse-engineering can be easily determined as the following.
Ideally, a TR map (e.g. the one in Fig. \ref{fig:outline}) is supposed to be dark almost everywhere except for a few entries that may be associated with the backdoor class pairs.
Thus, we start with a relatively large number of images (e.g. 15 or even more) to compute the TR statistics.
If there are more than $2(K^2-K)$ bright entries in the TR map with TR larger than some prescribed threshold, we reduce the number of images, e.g., by dividing it by 2.
Here, $K$ is the number of classes, and $K^2-K$ is the maximum number of entries in the TR map corresponding to a valid candidate set of backdoor class pairs.
The above steps are repeated until there are at most $2(K^2-K)$ bright entries in the TR map.

\begin{figure}[t!]
\vspace{-0.065in}
\includegraphics[width=.46\linewidth]{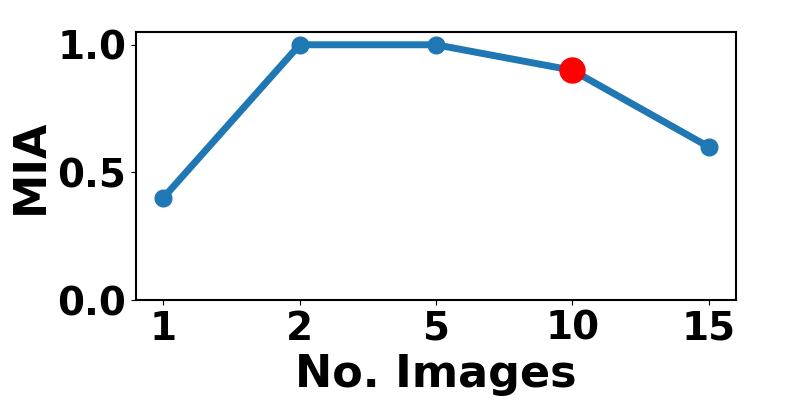}
\includegraphics[width=.46\linewidth]{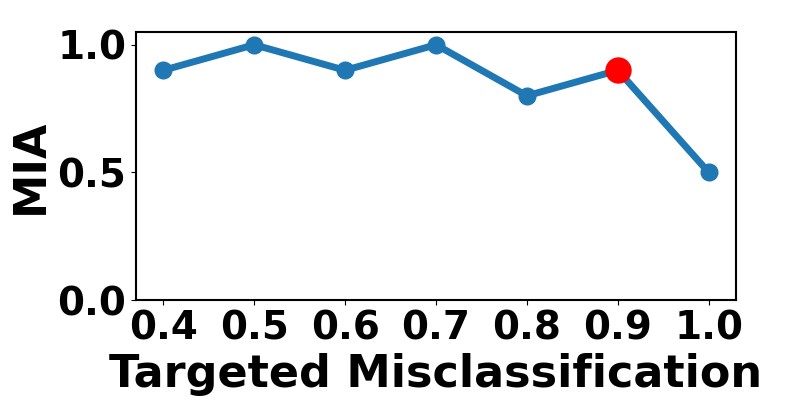}
\vspace{-0.095in}
\caption{Influence of the number of images and the targeted misclassification fraction used by \citeauthor{Post-TNNLS} (2020) for trigger reverse-engineering on our UMD. The default setting suggested by the authors, which is globally used in this work, is marked in red. UMD prefers even fewer images for trigger reverse-engineering and is insensitive to targeted misclassification fractions less than 1.}
\label{fig:ablation_ni_pi}
\vspace{-0.165in}
\end{figure}

\begin{table}[t!]
	\setlength{\tabcolsep}{4pt}
        \caption{MIA of UMD against a variety of X2X attacks with the WaNet trigger and the Blended trigger, respectively. UMD achieves generally high MIAs against all these X2X attacks for both triggers.}
	\centering{
		\scalebox{0.83}{%
			\begin{tabular}{c|cccc}
				\toprule \hline
				& 2to2 & 5to5 & 8to8 & A2Ar\\ \hline
				WaNet & 0.80 & 0.80 & 1.0 & 0.90\\ \hline
                    Blended & 0.90 & 0.90 & 0.90 & 0.90\\
            \hline\bottomrule
			\end{tabular}
		}
	}
	\label{tab:detection_advanced_trigger}
 \vspace{-0.195in}
\end{table}

\begin{table}[t!]
	\setlength{\tabcolsep}{4pt}
        \caption{UMD achieves generally high MIAs against four adaptive attacks, ISSBA, CLA, N2N, and O2O, on CIFAR-10.}
	\centering{
		\scalebox{0.83}{%
			\begin{tabular}{c|cccc}
				\toprule \hline
				& ISSBA & CLA & N2N & O2O\\ \hline
				MIA for UMD & 1.0  & 0.60 & 1.0 & 0.80\\
            \hline\bottomrule
			\end{tabular}
		}
	}
	\label{tab:detection_advanced}
 \vspace{-0.165in}
\end{table}

\vspace{-0.03in}
\subsection{Performance of UMD against Adaptive Attacks}\label{subsec:exp_adaptive}
\vspace{-0.03in}

Here, we show the detection performance of UMD against two advanced trigger types, WaNet \cite{nguyen2021wanet} and Blended \cite{Targeted}, for a variety of X2X attack settings.
We also evaluate UMD against four adaptive attacks, including the invisible sample-specific backdoor attack (ISSBA) proposed by \citeauthor{li_ISSBA_2021} (2021), the clean label attack (CLA) proposed by \citeauthor{Clean_Label_BA} (2019), the N2N attack proposed by \citeauthor{xue2022imperceptible} (2022), and the O2O attack with one randomly selected backdoor class pair.
We consider the default A2O setting for ISSBA and CLA since these two attacks cannot be easily extended to other X2X settings.
For each N2N attack, we launch $N=3$ A2O attacks together, each with a randomly selected target class and a random patch trigger.
The experiments in this section are conducted on CIFAR-10.
For each setting considered for each trigger or attack type, we create 10 attacks and train a model for each attack using the configurations in Sec. \ref{subsec:exp_main_setup}.

Due to the complexity of the trigger embedding functions for WaNet and Blended, we employ a more general trigger reverse-engineering algorithm proposed by \citeauthor{Post-TNNLS} (2020), which estimates a {\it common} additive perturbation in the internal layer of the classifier (see Apdx. \ref{subsubsec:exp_main_supp_review_re_ss} for more details).
For the N2N attack, we introduce a trivial generalization of  UMD by sequentially selecting multiple clusters (by repeating lines 5-13 of Alg. \ref{alg:graph_opt} multiple times).
Each cluster is then inferred by the same anomaly detection procedure in Sec. \ref{subsubsec:detection_inference}, where the ``null'' statistics are those not belonging to any clusters.
Intuitively, these clusters will either be associated with one of the $N$ triggers or be the non-backdoor class pairs and rejected by anomaly detection.

In Tab. \ref{tab:detection_advanced_trigger}, we show the effectiveness of UMD against the WaNet trigger and the Blended trigger for a variety of X2X attacks.
In Tab. \ref{tab:detection_advanced}, we show that UMD can also detect the four adaptive attacks with generally high MIA.
Notably, although UMD always selects at least two putative backdoor class pairs for inference, it still detects the O2O attack (with only one backdoor class pair) well, thanks to the (almost inevitable) collateral damage which introduces additional ``backdoor class pairs'' (see Apdx. \ref{subsec:others_o2o_collateral} for more details).
Moreover, for the 10 N2N attacks, the generalized UMD that selects 5 clusters of candidate backdoor class pairs correctly identifies 28 out of the 10x3 triggers, with only 2 clusters falsely recognized as associated with the backdoor.

\begin{table}[t!]
\vspace{-0.175in}
\setlength{\tabcolsep}{4pt}
\caption{Using the class pairs detected by UMD to mitigate the 2to2, 5to5, 8to8, and A2Ar attacks on CIFAR-10, based on the method by \citeauthor{NC} (2019). All the backdoored classifiers are ``fixed'' as reflected by the low average ASR (\%) with negligible degradation in the average ACC (\%).}
\centering{
    \scalebox{0.82}{%
        \begin{tabular}{c|cccc}
            \toprule \hline
            & 2to2 & 5to5 & 8to8 & A2Ar \\ \hline
            ASR (Avg) & 98.1$\rightarrow$1.4 & 93.3$\rightarrow$1.4 & 91.2$\rightarrow$7.2 & 89.9$\rightarrow$11.2 \\ \hline
            ACC (Avg) & 92.4$\rightarrow$92.2 & 92.7$\rightarrow$92.3 & 92.8$\rightarrow$92.3 & 93.7$\rightarrow$91.9\\
            \hline\bottomrule
        \end{tabular}
    }
}
\label{tab:mitigation}
\vspace{-0.21in}
\end{table}

\vspace{-0.065in}
\subsection{Backdoor Mitigation}\label{subsec:exp_mitigation}
\vspace{-0.045in}

The backdoor class pairs detected by UMD can be used to ``fix'' the backdoored model.
This process is called \textit{backdoor mitigation} or \textit{Trojan removal} \cite{TRC}.
Here, we use the method proposed by \citeauthor{NC} (2019) to mitigate the 2to2, 5to5, 8to8, and A2Ar attacks on CIFAR-10 that are detected by UMD.
For each class pair being detected, we embed the reverse-engineered trigger into clean samples from the source class but without changing their labels.
By fine-tuning using these samples, together with some clean samples without the trigger (to maintain the ACC), the model will learn to predict correctly even if a test sample is embedded with the trigger, i.e. the backdoor will be ``unlearned''.
This is shown in Tab. \ref{tab:mitigation}, where for all attack settings, the average ASR drops to $\leq11.2\%$ with negligible degradation in the average ACC -- the models are fixed.

\vspace{-0.095in}
\section{Conclusion}\label{sec:conclusion}
\vspace{-0.045in}

We proposed UMD, the first unsupervised backdoor model detector against X2X attacks. We defined TR and proved its intrinsic property in distinguishing backdoor class pairs from non-backdoor class pairs. Our UMD first selects a set of putative backdoor class pairs based on the TR statistics by solving a clustering problem we proposed, and then uses a robust, unsupervised anomaly detector to infer both the presence of the attack and the backdoor class pairs. Empirically, we show that UMD performs well on three datasets against X2X attacks with diverse settings.

\noindent\textbf{Acknowledgements} This work is partially supported by the NSF grant No.1910100, No. 2046726, Defense Advanced Research Projects Agency (DARPA) No. HR00112320012, C3.ai, and Amazon Research Award.

\newpage

\nocite{langley00}

\bibliography{ref}
\bibliographystyle{icml2023}

%%%%%%%%%%%%%%%%%%%%%%%%%%%%%%%%%%%%%%%%%%%%%%%%%%%%%%%%%%%%%%%%%%%%%%%%%%%%%%%
%%%%%%%%%%%%%%%%%%%%%%%%%%%%%%%%%%%%%%%%%%%%%%%%%%%%%%%%%%%%%%%%%%%%%%%%%%%%%%%
% APPENDIX
%%%%%%%%%%%%%%%%%%%%%%%%%%%%%%%%%%%%%%%%%%%%%%%%%%%%%%%%%%%%%%%%%%%%%%%%%%%%%%%
%%%%%%%%%%%%%%%%%%%%%%%%%%%%%%%%%%%%%%%%%%%%%%%%%%%%%%%%%%%%%%%%%%%%%%%%%%%%%%%
\newpage

\appendix
\onecolumn

\section{Ethics Statement}\label{sec:ethics_statement}

The main purpose of this research is to understand the behavior of deep learning systems facing malicious activities and enhance their safety without degrading their utility.
The X2X backdoor attack considered in this paper is the union of many well-known backdoor attacks with different settings -- all these attacks are open-sourced.
Thus, our work will be beneficial to the community in defending against these attacks via detection.
However, we do not claim that our detector is effective against all backdoor attacks that may appear in the future.
In fact, there is no published backdoor detector making such a claim, just like that there is no published backdoor attack proved to be evasive against all future detectors.
The code related to this work can be found at: \url{https://github.com/polaris-73/MT-Detection}
Finally, the paper is written by humans without the involvement of large language models.

\section{Analysis of TR and Proofs}\label{sec:analysis}

Here, we present the complete analysis showing that the TR statistic is intrinsically suitable for detecting backdoor class pairs. Such an intrinsic property of TR is not possessed by many popular statistics for backdoor model detection. For example, the (patch) size of the reverse-engineered triggers used by \citeauthor{NC} (2019) is based on the premise that the actual trigger used by the attacker is small.

Our main theoretical results are summarized in Thm. \ref{thm:main} in Sec. \ref{subsec:mt} (also restated as Thm. \ref{thm:main_restate} below). Intuitively, the theorem says that the trigger reverse-engineered for a backdoor class pair will likely induce a small classification loss to all the other backdoor class pairs. Thus, empirically, we will likely observe a large TR statistic (possibly close to 1) from one backdoor class pair to another. In the following, we first present the complete problem settings that will facilitate our analysis. Then we prove Thm. \ref{thm:main}.

\subsection{Complete Settings}\label{subsec:analysis_settings}

{\bf Set of class pairs:} We consider an {\it arbitrary} set of class pairs ${\mathcal A}'=\{a_1, \cdots, a_k\}$ ($k\leq|{\mathcal Y}|$) satisfying:
\begin{itemize}
\item For $\forall a=(s, t)\in{\mathcal A}'$, $s\neq t$ (i.e. condition (1) in Def. \ref{def:target_specific_attack});
\item If $|{\mathcal A}'|>1$, for any $a_i=(s_i, t_i)\in{\mathcal A}'$ and $a_j=(s_j, t_j)\in{\mathcal A}'$, $s_i\neq s_j$ if $a_i\neq a_j$ (i.e. condition (2) in Def. \ref{def:target_specific_attack});
\item $P_A(a)>0$ for $\forall a\in{\mathcal A}'$ (i.e. positive probability for all class pairs in ${\mathcal A}'$).
\end{itemize}
Note that here, we do not specify if any class pair $a\in{\mathcal A}'$ is a backdoor class pair or not.

{\bf Random variables:} Following the main paper, we use $X\in{\mathcal X}$ and $Y\in{\mathcal Y}$ to denote the random variables for samples and labels respectively. $A\in{\mathcal A}'$ denotes the random variable for class pairs in ${\mathcal A}'$. Moreover, for any trigger embedding function $\delta$, we use $X^{\delta}\triangleq\delta(X)$ to denote the random variable for samples generated from $X$ by embedding a trigger using $\delta$. Then, each $\delta$ specifies a conditional distribution $P_{X^{\delta}|X}$. In summary of the above, we have the following dependency:
\begin{equation}\label{eq:dependency}
	(A, X, Y, X^{\delta})\sim P_A \cdot P_{XY|A} \cdot P_{X^{\delta}|X}
\end{equation}

{\bf Set of estimators/classifiers:} Considering that TR is defined in terms of the (expected) classification loss on samples with a (reverse-engineered) trigger embedded (see Eq. \eqref{eq:mt_definition}), we use ${\mathcal F}$ to represent the set of estimators (i.e. classifiers in our problem) for estimating $Y$ from the trigger-embedded sample $X^{\delta}$ with arbitrary $\delta$. For example, ${\mathcal F}$ may contain all classifiers with the same architecture as the one to be inspected (i.e. the classifier that will also be used for trigger reverse-engineering) but with different parameter values. For convenience, we also define $\Delta$ as the set of all trigger embedding functions. For example, for image perturbation triggers, $\Delta$ may include perturbations with different shapes and sizes. For another example, for sample-specific triggers, $\Delta$ may be the set of all autoencoders with the same architecture but different parameter values. Moreover, we define a set ${\mathcal G}\triangleq\Delta\times{\mathcal F}$ of ``end-to-end'' functions, such that each $g\in{\mathcal G}$ can be represented by $g=f\circ\delta$ for some $\delta\in\Delta$ and $f\in{\mathcal F}$. These sets of estimators and their relation to the random variables we have defined previously are illustrated in Fig. \ref{fig:estimators}.

\begin{figure}[h]
	\centering
	\includegraphics[width=.24\columnwidth]{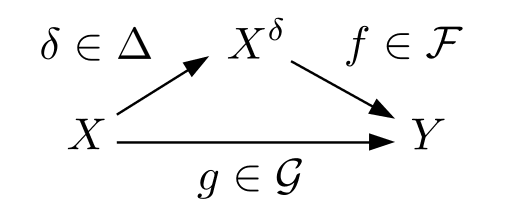}
	\caption{Illustration of the estimators/classifiers and their relation to the random variables.}\label{fig:estimators}
\end{figure}
 
{\bf Bayes classifiers:} Bayes classifier refers to the classifier with the minimum classification loss when predicting/estimating the label of a random input sample \cite{bayes}. Typically, the Bayes classifier (usually with respect to a space of classifiers) is specified by the joint distribution of the input and the label. For example, in the main paper, given a trigger embedding function $\delta$, we denote the (set of) Bayes classifier(s) for estimating $Y$ from $X^{\delta}$ (with joint distribution $P_{X^{\delta}Y}$) as ${\mathcal F}^{\delta}\subset{\mathcal F}$ (see Eq. \eqref{eq:bayes_classifier}). And we denote the associated Bayes risk as $R^{\mathcal F}(Y|X^{\delta})$ (see Eq. \eqref{eq:bayes_risk}). Here, among all ``end-to-end'' classifiers in ${\mathcal G}$ for estimating $Y$ from $X$ (by first embedding a trigger and then classifying), where $X, Y\sim P_{XY}$, we denote the set of Bayes classifiers (i.e. with the smallest classification loss) as:
\begin{equation}\label{eq:bayes_classifier_g}
	{\mathcal G}^{\ast} = \{g\in{\mathcal G} | \mathbb{E}_{P_{XY}} [l(Y, g(X))]=R^{\mathcal G}(Y|X)\}
\end{equation}
where
\begin{equation}\label{eq:bayes_risk_g}
	R^{\mathcal G}(Y|X) = \min_{g\in{\mathcal G}} \mathbb{E}_{P_{XY}} [l(Y, g(X))]
\end{equation}
denotes the associated Bayes risk. Similarly, for each class pair $a\in{\mathcal A}'$ with conditional joint distribution $P_{XY|a}$ for sample $X$ and label $Y$, we denote the set of Bayes classifiers, with respect to the set ${\mathcal G}$, for estimating $Y$ from $X$ given $a$ as:
\begin{equation}\label{eq:bayes_classifier_g_conditional}
	{\mathcal G}_a^{\ast} = \{g\in{\mathcal G} | \mathbb{E}_{P_{XY|a}} [l(Y, g(X))]=R^{\mathcal G}_a(Y|X)\}
\end{equation}
where
\begin{equation}\label{eq:bayes_risk_g_conditional}
	R^{\mathcal G}_a(Y|X) = \min_{g\in{\mathcal G}} \mathbb{E}_{P_{XY|a}} [l(Y, g(X))]
\end{equation}
is the associated Bayes risk with conditioning on $a$. Finally, for each class pair $a\in{\mathcal A}'$ and any $\delta\in\Delta$, the set of Bayes classifiers, with respect to ${\mathcal F}$, for estimating $Y$ from $X^{\delta}$ (both conditioned on $a$) can be written as:
\begin{equation}\label{eq:bayes_classifier_conditional}
	{\mathcal F}_a^{\delta} = \{f\in{\mathcal F} | \mathbb{E}_{P_{X^{\delta}Y|a}} [l(Y, f(X^{\delta}))]=R_a^{\mathcal F}(Y|X^{\delta})\}.
\end{equation}
where 
\begin{equation}\label{eq:bayes_risk_conditional}
	R_a^{\mathcal F}(Y|X^{\delta}) = \min_{f\in{\mathcal F}} \mathbb{E}_{P_{X^{\delta}Y|a}} [l(Y, f(X^{\delta}))]
\end{equation}
is the associated Bayes risk conditioned on $a$.

\subsection{Proof of Thm. \ref{thm:main}}

To begin with, we show a mild assumption required by the theorem:
\begin{assumption}\label{asm:commonality}
	{\it $\exists g\in{\mathcal G}^{\ast}$ satisfying $g\in{\mathcal G}_a^{\ast}$ for $\forall a\in{\mathcal A}'$}.
\end{assumption}

{\bf Remarks:} The assumption basically says that there exists a Bayes classifier $g$ for estimating $Y$ from $X$ (unconditionally) that is also Bayes when $X$ and $Y$ are both conditioned on some arbitrary class pair $a\in{\mathcal A}'$. For convenience, we define ${\mathcal B}=\{(s, t)\in{\mathcal Y}\times{\mathcal Y}|s=t\}$ as the set of all ``identical'' pairs. Then, the assumption is {\it guaranteed to hold} if the samples together with their (correct) labels following the joint distribution $P_{XY|{\mathcal B}}$ are {\it perfectly separable} by some classifier $f\in{\mathcal F}$. To see this, let's first consider the case where ${\mathcal A}'={\mathcal B}$. We can easily construct the desired function $g=f\circ \delta$ from $f$, with $\delta$ being an identity mapping. Then, given that ${\mathbb E}_{P_{XY|{\mathcal B}}}[l(Y, f(X))]=0$ for $l$ being the 0-1 loss (which is due to that $X, Y\sim P_{XY|{\mathcal B}}$ is perfectly separable by $f$), ${\mathbb E}_{P_{XY|{\mathcal B}}}[l(Y, g(X))]=0$ will also hold since $\delta(X)=X$ by our construction. Since the loss is defined to be non-negative, we will then have ${\mathbb E}_{P_{XY|a}}[l(Y, g(X))]=0$ for $\forall a\in{\mathcal A}'={\mathcal B}$. Next, we consider the case where ${\mathcal A}'\neq{\mathcal B}$. We first construct an injective mapping $\phi:{\mathcal A}'\rightarrow{\mathcal B}$, such that for any $a=(s,t)\in{\mathcal A}'$ and $a'=(s',t')=\phi(a)\in{\mathcal B}$ (with $s'=t'$ by the definition of ${\mathcal B}$), $a'=\phi(a)$ {\it if and only if} $s=s'$. The existence of such $\phi$ is guaranteed by that: (a) both ${\mathcal A}'$ and ${\mathcal B}$ satisfy condition (1) in Def. \ref{def:target_specific_attack} (see the definition of ${\mathcal A}'$ in Sec. \ref{subsec:analysis_settings}), and (b) $|{\mathcal A}'|\leq|{\mathcal B}|=|{\mathcal Y}|$ (which allows each element in ${\mathcal A}'$ to have an image in ${\mathcal B}$). Thus, we can easily rearrange the output neurons of $f$ based on the mapping $\phi$. In particular, for any $a=(s,t)\in{\mathcal A}'$ and its associated $a'=(s',t')=\phi(a)$, we relabel class $t'$ (where $t'=s'=s$) to class $t$. If two different class pairs $a_i=(s_i,t)\in{\mathcal A}'$ and $a_j=(s_j,t)\in{\mathcal A}'$ share the same target class $t$, the rearranged classifier will predict to class $t$ if $f$ predicts to any of $s_i$ and $s_j$. Then, we will also obtain a desired classifier $g$ satisfying Assumption \ref{asm:commonality} by affiliating an identity trigger embedding function $\delta$ to the classifier rearranged from $f$ following the procedure above.

In the proof of Thm. \ref{thm:main}, we will also need the following lemmas.

\begin{lemma}\label{lem:gdpi}
	({\bf Generalized Data Processing Inequality} \cite{GDPI}) Suppose random variables $X^{\delta}$ and $Y$ are conditionally independent given $X$. Then, for any loss function $l$, we have:
	\begin{equation*}
		R^{\mathcal F}(Y|X^{\delta}) \geq R^{\mathcal G}(Y|X).
	\end{equation*}
\end{lemma}

\begin{lemma}\label{lem:gdpi_equality}
	There always exists $\delta$ such that
	\begin{equation*}
		R^{\mathcal F}(Y|X^{\delta}) = R^{\mathcal G}(Y|X).
	\end{equation*}
	Moreover, for each $a\in{\mathcal A}'$, there also exists $\delta$ such that
	\begin{equation*}
		R_a^{\mathcal F}(Y|X^{\delta}) = R_a^{\mathcal G}(Y|X).
	\end{equation*}
\end{lemma}

\begin{proof}
	For the unconditional case, we construct $\delta=f^{-1}\circ g^{\ast}$ with arbitrary $f\in{\mathcal F}$ and arbitrary $g^{\ast}\in{\mathcal G}^{\ast}$, such that $X^{\delta}=\delta(X)=f^{-1}(g^{\ast}(X))$. Then, we have
	\begin{align*}
		R^{\mathcal G}(Y|X) &= \mathbb{E}_{P_{XY}} [l(Y, g^{\ast}(X))] \tag*{$\triangleright$ {Eq. \eqref{eq:bayes_classifier_g} and \eqref{eq:bayes_risk_g}}}\\
		&= \mathbb{E}_{P_{XY}\cdot P_{X^{\delta}|X}} [l(Y, f(X^{\delta}))] \tag*{$\triangleright$ Construction of $\delta$}\\
		&\geq R^{\mathcal F}(Y|X^{\delta}) \tag*{$\triangleright$ {Eq. \eqref{eq:bayes_classifier} and \eqref{eq:bayes_risk}}}
	\end{align*}
	According to Lemma \ref{lem:gdpi}, since $X^{\delta}$ and $Y$ are indeed conditionally independent given $X$, equality must hold in above for the constructed $\delta$.
	
	For the conditional case and for each $a\in{\mathcal A}'$, a similar proof can be applied with $\delta$ constructed by choosing $g^{\ast}$ from ${\mathcal G}_a^{\ast}$.
\end{proof}

\begin{lemma}\label{lem:sandwitch}
	If $\delta$ minimizes $R^{\mathcal F} (Y|X^{\delta})$, then, for any $a\in{\mathcal A}'$:
	(1) $R_a^{\mathcal F}(Y|X^{\delta})=R_a^{\mathcal G}(Y|X)$;
	(2) ${\mathcal F}^{\delta}\subset{\mathcal F}_a^{\delta}$.
\end{lemma}

\begin{proof}
	Considering an arbitrary $f^{\ast}\in{\mathcal F}^{\delta}$ and an arbitrary $g^{\ast}\in{\mathcal G}^{\ast}$ satisfying $g^{\ast}\in{\mathcal G}_a^{\ast}$ for $\forall a\in{\mathcal A}'$ (existence of such $g^{\ast}$ is guaranteed by Assumption \ref{asm:commonality}), for the estimation of $Y$ from both $X$ and $X^{\delta}$, we have the following relationship between the Bayes risks {\it with} and {\it without} conditioning:
	\begin{align*}
		R^{\mathcal G}(Y|X) &= \mathbb{E}_{P_{XY}} [l(Y, g^{\ast}(X))] \tag*{$\triangleright$ {Eq. \eqref{eq:bayes_classifier_g} and \eqref{eq:bayes_risk_g}}}\\
		&= \sum_{a\in{\mathcal A}'} P_A(a) \mathbb{E}_{P_{XY|a}} [l(Y, g^{\ast}(X))] \tag*{$\triangleright$ Conditioning}\\
		&= \sum_{a\in{\mathcal A}'} P_A(a) R_a^{\mathcal G}(Y|X) \tag*{$\triangleright$ Eq. \eqref{eq:bayes_classifier_g_conditional} and \eqref{eq:bayes_risk_g_conditional}}\\
		R^{\mathcal F}(Y|X^{\delta}) &= \mathbb{E}_{P_{X^{\delta}Y}} [l(Y, f^{\ast}(X^{\delta}))] \tag*{$\triangleright$ {Eq. \eqref{eq:bayes_classifier} and \eqref{eq:bayes_risk}}}\\
		&= \sum_{a\in{\mathcal A}'} P_A(a) \mathbb{E}_{P_{X^{\delta}Y|a}} [l(Y, f^{\ast}(X^{\delta}))] \tag*{$\triangleright$ Conditioning}\\
		&\geq \sum_{a\in{\mathcal A}'} P_A(a) R_a^{\mathcal F}(Y|X^{\delta}) \tag*{$\triangleright$ Eq. \eqref{eq:bayes_classifier_conditional} and \eqref{eq:bayes_risk_conditional}}\\
	\end{align*}
	Combining the above, we have:
	\begin{align}
		R^{\mathcal F}(Y|X^{\delta}) - R^{\mathcal G}(Y|X) &\geq \sum_{a\in{\mathcal A}'} P_A(a) (R_a^{\mathcal F}(Y|X^{\delta}) - R_a^{\mathcal G}(Y|X)) \label{eq:conditioning}\\
		&\geq 0 \tag*{$\triangleright$ Lemma \ref{lem:gdpi}}
	\end{align}
	Since that $\delta$ minimizes $R^{\mathcal F} (Y|X^{\delta})$ is given, by Lemma \ref{lem:gdpi} and Lemma \ref{lem:gdpi_equality}, we have $R^{\mathcal F}(Y|X^{\delta}) - R^{\mathcal G}(Y|X)=0$. Thus, the inequalities above both become equality. Since $P_A(a)>0$ for $\forall a\in{\mathcal A}'$ (see the settings of ${\mathcal A}'$ in Sec. \ref{subsec:analysis_settings}), item (1) of the lemma, i.e. $R_a^{\mathcal F}(Y|X^{\delta})=R_a^{\mathcal G}(Y|X)$ for $\forall a\in{\mathcal A}'$, is proved.
	
	Next, we prove item (2) of the lemma by contradiction. Assume that there exist $a\in{\mathcal A}'$ and $f'\in{\mathcal F}^{\delta}$ such that $f'\notin{\mathcal F}_a^{\delta}$. Then,
	\begin{align*}
		R^{\mathcal F}(Y|X^{\delta}) &= \mathbb{E}_{P_{X^{\delta}Y}} [l(Y, f'(X^{\delta}))] \tag*{$\triangleright$ {Eq. \eqref{eq:bayes_classifier} and \eqref{eq:bayes_risk}}}\\
		&= \sum_{a'\in{\mathcal A}'\setminus a} P_A(a') \mathbb{E}_{P_{X^{\delta}Y|a'}} [l(Y, f'(X^{\delta}))] + P_A(a) \mathbb{E}_{P_{X^{\delta}Y|a}} [l(Y, f'(X^{\delta}))] \tag*{$\triangleright$ Conditioning}\\
		&> \sum_{a'\in{\mathcal A}'\setminus a} P_A(a') R_{a'}^{\mathcal F}(Y|X^{\delta}) + P_A(a) R_{a}^{\mathcal F}(Y|X^{\delta}) \tag*{$\triangleright$ Eq. \eqref{eq:bayes_risk_conditional} and $f'\notin{\mathcal F}_a^{\delta}$}\\
		&= \sum_{a'\in{\mathcal A}'} P_A(a') R_{a'}^{\mathcal F}(Y|X^{\delta})
	\end{align*}
	Thus, the inequality \eqref{eq:conditioning} becomes strict and moreover, $R^{\mathcal F}(Y|X^{\delta}) - R^{\mathcal G}(Y|X)>0$. Here, we have reached a contradiction since $R^{\mathcal F}(Y|X^{\delta}) - R^{\mathcal G}(Y|X)=0$ must hold when $\delta$ minimizes $R^{\mathcal F} (Y|X^{\delta})$ as discussed above.
\end{proof}

\begin{theorem}\label{thm:main_restate}
	({\bf Restatement of Thm. \ref{thm:main}}) {\it For any class pair $a\in{\mathcal A}'$, consider a trigger embedding function $\delta$ that minimizes $R^{\mathcal F}_a(Y|X^{\delta})$. Then, $\delta$ minimizes:
	\begin{equation*}
		\min_{f\in{\mathcal F}_a^{\delta}} \sum_{a'\in{\mathcal A}'\setminus a} P_{A|A\neq a}(a') \mathbb{E}_{P_{X^{\delta}Y|a'}} [l(Y, f(X^{\delta}))]
	\end{equation*}
	if and only if $\delta$ also minimizes $R^{\mathcal F} (Y|X^{\delta})$.}
\end{theorem}

\begin{proof}
	For any $\delta\in\Delta$ and $a\in{\mathcal A}'$, we have the following lower bound for the minimum:
	\begin{align*}
		&\min_{f\in{\mathcal F}_a^{\delta}} \sum_{a'\in{\mathcal A}'\setminus a} P_{A|A\neq a}(a') \mathbb{E}_{P_{X^{\delta}Y|a'}} [l(Y, f(X^{\delta}))] \\
		\geq &\sum_{a'\in{\mathcal A}'\setminus a} P_{A|A\neq a}(a') R_{a'}^{\mathcal F}(Y|X^{\delta}) \tag*{$\triangleright$ Eq. \eqref{eq:bayes_risk_conditional} \quad $(\ast)$}\\
		\geq &\sum_{a'\in{\mathcal A}'\setminus a} P_{A|A\neq a}(a') R_{a'}^{\mathcal G}(Y|X) \tag*{$\triangleright$ Lemma \ref{lem:gdpi} \quad $(\ast\ast)$}
	\end{align*}
	{\bf Proof of sufficiency} We show that if $\delta$ minimizes $R^{\mathcal F} (Y|X^{\delta})$, the lower bound above will be reached, i.e. equality holds for both $(\ast)$ and $(\ast\ast)$. First, by item (2) of Lemma \ref{lem:sandwitch}, there exist $f^{\ast}\in{\mathcal F}_a^{\delta}$ satisfying $f^{\ast}\in{\mathcal F}_{a'}^{\delta}$ for $\forall a'\in{\mathcal A}'\setminus a$. Thus, based on Eq. \eqref{eq:bayes_classifier_conditional}, equality holds for $(\ast)$. Next, by item (1) of Lemma \ref{lem:sandwitch}, equality holds for $(\ast\ast)$.
	
	{\bf Proof of necessity} We prove by contradiction. Suppose $\delta$ {\it does not} minimize $R^{\mathcal F} (Y|X^{\delta})$, by Lemma \ref{lem:gdpi} and Lemma \ref{lem:gdpi_equality}, we will have:
	\begin{equation*}
		R^{\mathcal F}(Y|X^{\delta}) - R^{\mathcal G}(Y|X)>0
	\end{equation*}
	Then, based on inequality \eqref{eq:conditioning}, {\it at least one} of the following must hold:
	\begin{equation*}
		\begin{aligned}
			\text{(A)} \quad & R^{\mathcal F}(Y|X^{\delta}) > \sum_{a'\in{\mathcal A}'} P_A(a') R_{a'}^{\mathcal F}(Y|X^{\delta})\\
			\text{or \quad (B)} \quad & R_{a'}^{\mathcal F}(Y|X^{\delta}) - R_{a'}^{\mathcal G}(Y|X) > 0 \quad \text{for some} \,\, a'\in{\mathcal A}'
		\end{aligned}
	\end{equation*}
	If (B) holds, we will further have $R_{a'}^{\mathcal F}(Y|X^{\delta}) - R_{a'}^{\mathcal G}(Y|X) > 0$ for some $a'\in{\mathcal A}'\setminus a$. This is because for the given $a$, $R_{a}^{\mathcal F}(Y|X^{\delta}) - R_{a}^{\mathcal G}(Y|X) = 0$ due to both that $\delta$ minimizes $R_a^{\mathcal F}(Y|X^{\delta})$ and the existence of such minimum (based on Lemma \ref{lem:gdpi_equality}). Then, equality {\it cannot} be achieved for $(\ast\ast)$ and we have reached a contradiction.
	
	But if (B) does not hold, (A) must hold. Again, for $(\ast)$ being equal, there must exist $f^{\ast}\in{\mathcal F}_a^{\delta}$ satisfying $f^{\ast}\in{\mathcal F}_{a'}^{\delta}$ for $\forall a'\in{\mathcal A}'\setminus a$. In other words, there exists (at least one) $f^{\ast}\in\cup_{a'\in{\mathcal A}'}{\mathcal F}_{a'}^{\delta}\neq\emptyset$. Thus, we have:
	\begin{align*}
		R^{\mathcal F}(Y|X^{\delta}) &\leq \mathbb{E}_{P_{X^{\delta}Y}} [l(Y, f^{\ast}(X^{\delta}))] \tag*{$\triangleright$ Eq. \eqref{eq:bayes_risk}}\\
		&= \sum_{a'\in{\mathcal A}'} P_A(a') \mathbb{E}_{P_{X^{\delta}Y|a'}} [l(Y, f^{\ast}(X^{\delta}))] \tag*{$\triangleright$ Conditioning}\\
		&= \sum_{a'\in{\mathcal A}'} P_A(a') R_{a'}^{\mathcal F}(Y|X^{\delta}) \tag*{$\triangleright$ Eq. \eqref{eq:bayes_risk_conditional}}
	\end{align*}
	This is a clear contradiction with (A).
\end{proof}

\newpage
\section{Supplementary of the Main Experiments on Backdoor Model Detection}\label{sec:exp_main_supp}

\subsection{Details for the Datasets}\label{subsec:exp_main_supp_dataset}

CIFAR-10 is a benchmark dataset with $32\times 32$ color images from 10 classes for different categories of objects \cite{CIFAR10}. The training set contains 50,000 images and the test set contains 10,000 images, both evenly distributed in the 10 classes.

GTSRB is an image dataset for German traffic signs from 43 classes \cite{GTSRB}. The training set and the test set contain 39,209 and 12,630 images respectively. The image sizes vary in a relatively large range. Thus, we resize all the images to $32\times 32$ in our experiments for convenience.

Imagenette consists of $224\times 224$ color images from ten selected classes of the ImageNet dataset \cite{ImageNet} that are easily classified. The training set and the test set contain 9,469 and 3,925 images respectively.

\subsection{Details for the Backdoor Triggers}\label{subsec:exp_main_supp_trigger}

In our experiments in Sec. \ref{sec:exp}, we considered a global, perturbation-based trigger with a big `X' shape (dubbed ``Pert''), and a local patch trigger (dubbed ``Patch''). The Pert trigger is generated by positively perturbing each pixel on both diagonals of the image by the same perturbation size for all three color channels. For CIFAR-10, GTSRB, and Imagenette, we set the perturbation size to 5/255, 15/255, and 15/255, respectively. For the Patch trigger, we replace a small area of the image (for all three channels) with an image patch with the same shape and size. For CIFAR-10, GTSRB, and Imagenette, we use $3\times3$, $2\times2$, and $8\times 8$ square patches respectively. For each attack, the location for the patch replacement and the color for each pixel in the patch are both randomly selected. Examples of both triggers and the image embedded with each trigger (compared with the original, trigger-free image) are shown in Fig. \ref{fig:trigger_example}.

\begin{figure}[h]
\centering
\begin{minipage}[b]{0.7\linewidth}
    \centering
    \includegraphics[width=.3\linewidth]{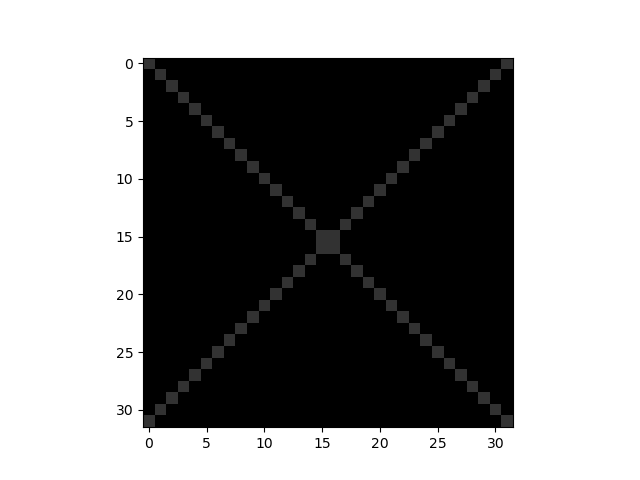}
    \includegraphics[width=.3\linewidth]{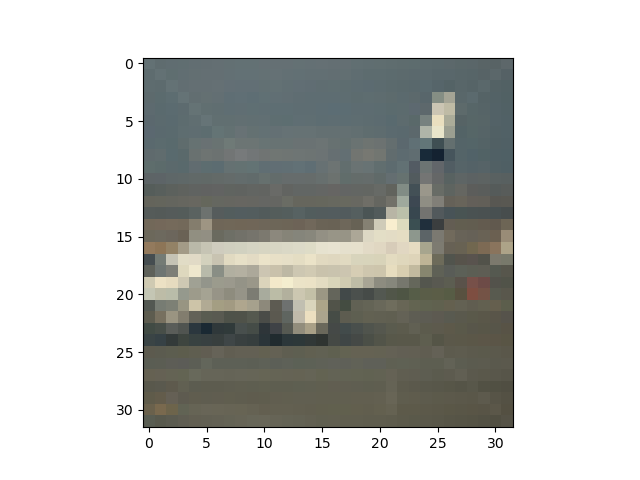}
    \includegraphics[width=.3\linewidth]{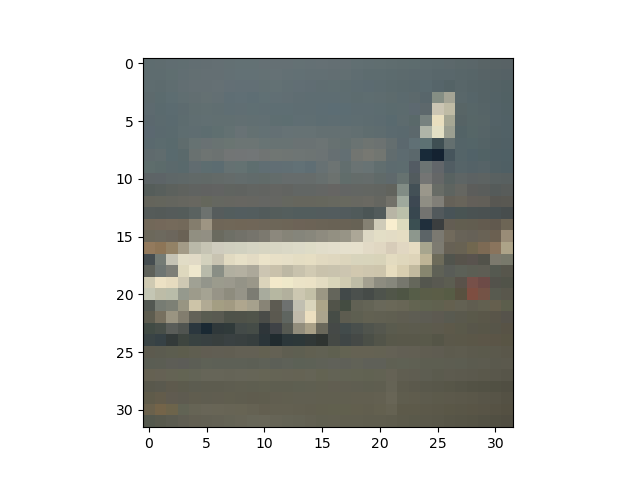}
\end{minipage}
\begin{minipage}[b]{0.7\linewidth}
    \centering
    \includegraphics[width=.3\linewidth]{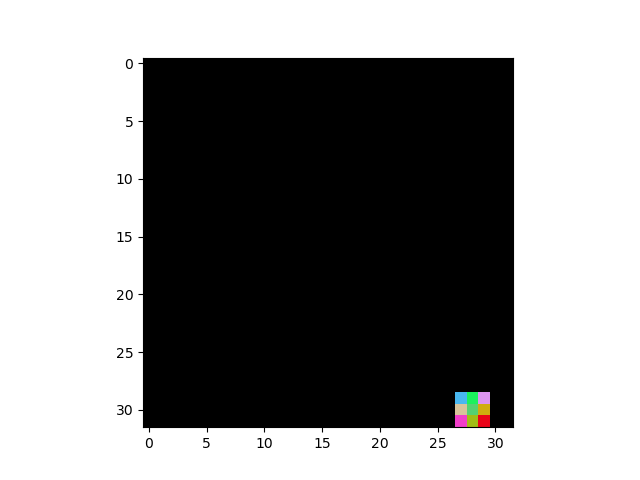}
    \includegraphics[width=.3\linewidth]{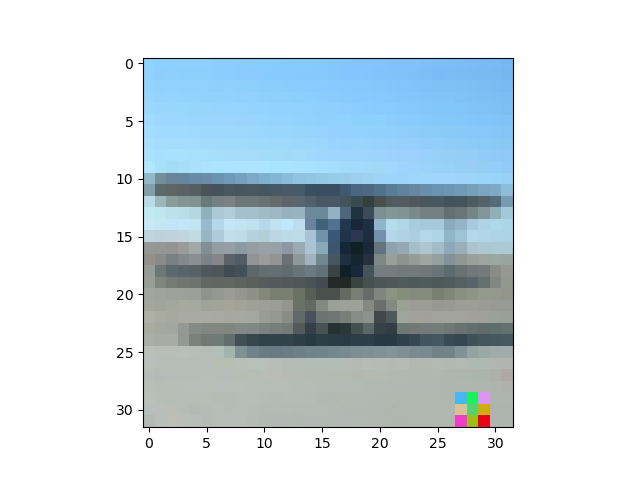}
    \includegraphics[width=.3\linewidth]{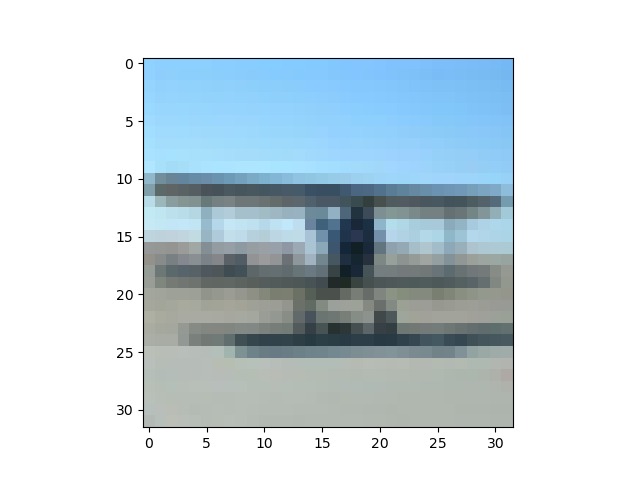}
\end{minipage}
\caption{Top: example of the Pert trigger (amplified to 50/255 perturbation size for better visualization), an image (from CIFAR-10) embedded with the Pert trigger (with perturbation size 5/255), and the original clean image without the trigger. Bottom: example of the Patch trigger, an image (also from CIFAR-10) embedded with the Patch trigger, and the original clean image without the trigger.}
\label{fig:trigger_example}
\end{figure}

\subsection{Training Configurations and Attack Effectiveness}\label{subsec:exp_main_supp_training}

For all three datasets, the training is performed on the training set specified in Apdx. \ref{subsec:exp_main_supp_dataset}. For CIFAR-10 and Imagenette, the training images are augmented by random horizontal flipping. For GTSRB, the training images are augmented by random rotation of $\pm 5$ degrees. We use ResNet-18 \cite{ResNet} as the model architecture for CIFAR-10 and Imagenette. For GTSRB, we use the model with the top performance on the leaderboard \cite{GTSRB_Leaderboard}. 
For CIFAR-10, GTSRB, and Imagenette, training is performed using the Adam optimizer \cite{Adam} for 200, 100, and 80 epochs, respectively, 
%\zidi{We are training 200 epochs, 100 epochs, and 80 epochs for CIFAR-10, GTSRB, and Imagenette respectively, } 
% For all three datasets, training is performed for 100 epochs 
%For all three datasets, training is performed 
with a learning rate of $10^{-3}$ and a mini-batch size of 64. 
When there is no attack, this training configuration achieves around 93\%, 98\%, and 88\% accuracy (ACC) for the three datasets, respectively.
The same set of configurations is also used for training the classifiers under the attacks we created.
The effectiveness of an attack is jointly measured by the ASR and ACC of the model.
The ASR for an X2X attack is the misclassification rate from the backdoor source classes to their designated target class when the samples from these source classes are embedded with the backdoor trigger.
In Tab. \ref{tab:asr_acc}, for each combination of the dataset, trigger, and attack setting, we show the average and the minimum ASR, together with the average and the minimum ACC for the ten classifiers we trained.
As a reference, the average and the minimum ACC for the ten benign classifiers for each dataset are also shown in Tab. \ref{tab:asr_acc}.

\begin{table}[t!]
\setlength{\tabcolsep}{4pt}
    \caption{The average (avg) and the minimum (min) ASR and ACC for each combination of the trigger and the attack setting on CIFAR-10, GTSRB, and Imagenette, and the average and the minimum ACC for the benign classifiers on each dataset for reference. The ASRs and the ACCs are all in percentage. All attacks we created are successful with ASR $\geq78\%$.}
\centering{
\scalebox{0.83}{%
\begin{tabular}{cc|ccccccc}
    \multicolumn{6}{c}{(a) CIFAR-10}\\
\toprule \hline
\multirow{1}*{Setting}  & Attack 
  &  avg ASR & min ASR & avg ACC & min ACC\\
    \hline
    Benign & - & -&- &93.54 &93.24\\
    \hline
    \multirow{2}*{A2O} & Patch & 99.73 & 98.08 & 93.08 &92.59\\
    &Pert&97.87 &95.87 &93.06 &92.76 \\\hline
    \multirow{2}*{2to2} & Patch & 97.90 & 95.80 & 92.76 & 91.70\\
    &Pert & 98.06 & 95.60 & 91.85 & 91.62 \\
    \hline
    \multirow{2}*{5to5} & Patch & 93.49 & 90.20 & 93.14 & 92.66\\
    &Pert & 93.10 & 87.12 & 91.91 & 91.48\\
    \hline
    \multirow{2}*{8to8} & Patch & 91.71 & 90.12 & 93.15 & 92.62\\
    &Pert & 90.06 & 87.85 & 91.96 & 91.34 \\
    \hline
    \multirow{2}*{A2Ar} & Patch&91.44 &89.88 &93.32 &92.92\\
    &Pert&87.27 &86.62 &93.35 &93.08 \\
    \hline \bottomrule
    \multicolumn{6}{c}{(b) GTSRB}\\
\toprule \hline
				\multirow{1}*{Setting}  & Attack 
  &  avg ASR & min ASR & avg ACC & min ACC\\
    \hline
    Benign & - & -&-& 98.46 & 98.27   \\
    \hline
    \multirow{2}*{A2O} & Patch & 99.99 & 99.93 & 98.12 & 97.69\\
    &Pert& 98.45 & 98.01 & 97.63 & 97.36 \\\hline
    \multirow{2}*{20to20} & Patch & 96.82 & 93.86 & 98.16 & 97.88\\
    &Pert & 95.49 & 93.92 & 98.05 & 97.60 \\
    \hline
    \multirow{2}*{30to30} & Patch & 95.85 & 91.75 & 98.13 & 97.77\\
    &Pert & 93.67 & 90.35 & 98.14 & 97.69 \\
    \hline
    \multirow{2}*{40to40} & Patch & 94.10 & 88.37 & 98.02 & 97.66\\
    &Pert & 93.58 & 92.11 & 98.10 & 97.78\\
    \hline
    \multirow{2}*{A2Ar} & Patch &94.44 & 93.37 & 98.14 &97.89\\
    &Pert & 93.55 & 92.13 & 98.26 & 98.04\\
    \hline \bottomrule
    \multicolumn{6} {c} {(c) ImageNette}\\
\toprule \hline
    \multirow{1}*{Setting}  & Attack 
  &  avg ASR & min ASR & avg ACC & min ACC\\
    \hline
    Benign & - & -&-& 88.81 & 87.95 \\
    \hline
    \multirow{2}*{A2O} & Patch & 99.51 & 99.14 & 88.36 & 87.57\\
    &Pert & 99.70 & 99.53 & 88.52 & 88.20 \\\hline
    \multirow{2}*{3to3} & Patch & 90.19 & 83.06 & 87.74 & 85.43\\
    &Pert & 92.66 & 89.71 & 88.36 & 87.85 \\
    \hline
    \multirow{2}*{5to5} & Patch & 81.04 & 78.26 & 87.46 & 85.89\\
    &Pert & 88.71 & 86.17 & 88.58 & 87.90 \\
    \hline
    \multirow{2}*{8to8} & Patch &- & - & - & -\\
    &Pert & 83.47 & 79.05 & 87.91 & 86.42\\
    \hline
    \multirow{2}*{A2Ar} & Patch &- & - & - & -\\
    &Pert & 82.31 & 81.20 & 88.17 & 87.44\\
    \hline \bottomrule
\end{tabular}
}
}
	\label{tab:asr_acc}
\end{table}

\subsection{Review of the Model Detection Methods Compared in Our Experiments}\label{subsec:exp_main_supp_review}

{\bf Neural Cleanse (NC)} is a typical reverse-engineering-based model detection method \cite{NC}.
It assumes an A2O attack and reverse-engineers a patch trigger with a size as small as possible for each putative target class using the algorithm described in Sec. \ref{subsubsec:exp_main_supp_review_re_patch}.
The premise behind NC is that the backdoor trigger will likely have a small size for human imperceptibility (which is generally true in practice), while the minimum size of a common patch that induces a large fraction of images to be misclassified to a non-backdoor target class will likely be large. 
With a trigger reverse-engineered for each class, NC adopts an unsupervised, MAD-based anomaly detector to infer if, for any class, the size of the reverse-engineered trigger is abnormally small based on a derived anomaly score.
The classifier is deemed to be attacked if the anomaly score is larger than a prescribed threshold (which indicates the existence of a reverse-engineered trigger with abnormally small size).
In our experiments, we use 20 clean images per class for detection and setting the threshold of the anomaly score to 2 \cite{NC}.
Note that this threshold, though claimed to be associated with a 95\% detection confidence level, implicitly assumes that the estimation of the MAD uses only a single null statistic, while the actual anomaly detection procedure of NC uses all the trigger size statistics for the estimation of MAD.
Moreover, threshold 2 is associated with the assumption that an anomaly may exist on both tails of the null distribution, i.e. both overly small and overly large trigger sizes are considered outliers, though a true detection should only be triggered by abnormally small trigger sizes (i.e. the small outliers).
Differently, our UMD determines a (single-tailed) confidence threshold based on the actual number of null statistics used for the estimation of MAD (see Sec. \ref{subsubsec:detection_inference}), which is more robust than NC to the changes of the domain size.
Note that based on Eq. \eqref{eq:threshold}, the same threshold 2 used by NC will be obtained if we set $N=1$ (for a single null statistic) and $\beta=0.025$ (for a single-tailed 0.025 significance level).
Despite the issue with the detection threshold, NC is not able to detect most X2X attacks except A2O attacks\footnote{A variant of NC with class-pair-wise trigger reverse-engineering was suggested by \citeauthor{NC} (2019) for detecting X2O attacks but without adequate evaluation on complicated datasets beyond MNIST \cite{MNIST}.} by design.
Moreover, NC is not implemented with class pair detection since once an attack is detected, all the class pairs with the target class being the detected target class will be treated as backdoor class pairs (by the definition of A2O attacks).

{\bf ABS} is also a reverse-engineering-based detector that assumes an A2O setting for potential attacks \cite{ABS}.
But before reverse-engineering the trigger, ABS first identifies a subset of neurons (e.g.) from the penultimate layer with the largest ``stimulation'' to particular neurons in the output layer.
That is, for any of these identified neurons, a large activation will subsequently lead to a large value for some neurons in the output layer.
Thus, ABS performs trigger reverse-engineering with a constraint to only boost the activation of these selected neurons.
The premise behind the design is that backdoor triggers will likely cause a large activation for some neurons in the intermediate layers.
Then, for each putative target class, the reverse-engineered trigger is embedded into a set of clean images and a REASR score is obtained as the misclassification fraction to the target class for these trigger-embedded images.
In the inference step, a larger REASR indicates that the classifier is more likely to be attacked.
Note that REASR is actually the ``transferability'' of the reverse-engineered trigger from one group of samples to another with respect to the same target class.
It is different from our TR statistic designed for each ordered pair of class pairs and does not endow ABS with the capability to detect general X2X attacks except A2O attacks.
In our experiments, we follow the descriptions in the original ABS paper by using one image per class and selecting 10 neurons from the penultimate layer of each classifier for detection.
For each putative target class, 30\% of the images are used for trigger reverse-engineering, and the remaining 70\% images are used to compute the REASR score.
Since ABS does not propose a practical method to select a threshold for the REASR score in an unsupervised fashion, in our experiments, {\it based on the resulting REASR scores}, we choose the threshold for ABS that keeps an approximately 95\% false detection rate across all three datasets (for a fair comparison with other methods adopting the same confidence level) while maximizing the overall true positive rate.

{\bf PT-RED} detects imperceptible, perturbation-based triggers by performing trigger reverse-engineering for each class pair. However, its inference step, which is based on probabilistic modeling with a threshold that controls the false detection rate, relies on the assumption of a single backdoor target class. Thus, it is capable of detecting X-to-one attacks with an inference of the source classes. But still, PT-RED cannot handle X2X attacks with more than one backdoor target class. In our experiments, we use 10 images per class for PT-RED and set the desired false detection rate to 5\% (i.e. 95\% confidence) based on the original paper.

{\bf MNTD} trains a binary\footnote{A one-class variant of MNTD is also proposed as a baseline by \citeauthor{META} (2021). However, the performance of this variant is not comparable to MNTD with the binary meta-classifier, thus is not evaluated in our experiments.} meta-classifier on features extracted from a large number of shadow classifiers with and without attack. Given a classifier to be inspected, features extracted from the classifier following the same procedure as for the shadow classifiers are fed into the meta-classifier to produce a score -- if a score is larger than a prescribed threshold, the classifier is deemed to be attacked, otherwise, it is not attacked. However, for the unsupervised model detection problem, a proper threshold is hard to choose. Thus, in our experiments, we choose a threshold based on the resulting scores to fix a 5\% false detection rate (i.e. 95\% confidence) while maximizing the true detection rate for a fair comparison with other methods. Moreover, since MNTD cannot cover the enormous space of attack settings for the X2X attack (see Fig. \ref{fig:venn_map}) when training the shadow classifiers, its effectiveness largely depends on the generalization capability of the attack settings (as well as the model architecture, the trigger, and so on) from the shadow classifiers to the actual classifier being attacked. Thus, based on the design, it is questionable for MNTD to detect X2X attacks with arbitrary settings.
In our experiments, we train a meta-classifier for each dataset using the code provided by the authors of META.
In particular, the shadow classifiers with the attack are trained in the A2O setting.
From our empirical results in Tab. \ref{tab:detection}, MNTD does not perform well even against A2O attacks, showing a poor generalization of the model architecture from the shadow models to the actual models to be inspected.
Finally, MNTD only performs model detection without the inference of backdoor class pairs.

{\bf K-Arm} focuses on solving the trigger reverse-engineering problem for each putative target class without knowing the actual number of source classes. Again, it well-addresses the X-to-one attacks, but cannot detect X2X attacks with more than one target class. Moreover, K-Arm uses the reverse-engineered trigger size for detection inference, which requires supervision for picking a threshold. In our experiments, we pick a threshold for K-Arm for each dataset and for each trigger type to control the false detection rate to 5\% while maximizing the true positive rate.
For the reverse-engineering step, we use 40 images per class. Finally, like all the other methods reviewed above, K-Arm is not implemented with backdoor class pair inference.

\subsection{Trigger Reverse-Engineering Algorithms}\label{subsec:exp_main_supp_review_re}

In this paper, we have considered three trigger reverse-engineering algorithms. In our experiments in Sec. \ref{sec:exp}, we equip UMD with the algorithms used by PT-RED \cite{Post-TNNLS} and NC \cite{NC} to address the perturbation trigger and the patch trigger respectively. In our experiments in Sec. \ref{subsec:exp_adaptive}, we show that UMD can even incorporate with the intermediate-layer trigger reverse-engineering technique \cite{Post-TNNLS} to address the stronger sample-specific backdoor attack. Here, we introduce these algorithms in detail.

\subsubsection{Reverse-Engineering Perturbation Triggers}\label{subsubsec:exp_main_supp_review_re_pert}

Perturbation triggers take the form $\delta(X)=[X+v]_{\rm c}$ for the embedding function where $v$ is a perturbation with a small $||v||_2$ for human imperceptibility and $[\cdot]_{\rm c}$ is a clipping function. Thus, reverse-engineering a perturbation trigger solves problem \eqref{eq:re_main} with $d(X, \delta(X))=||X-\delta(X)||_2\approx||v||_2$, i.e.:
\begin{equation*}
	\begin{aligned}
		& \underset{v}{\text{minimize}}
		& & ||v||_2\\
		& \text{subject to}
		& & v \in \argmin_{v'} {\mathbb E}_{P_{XY|a}}[l(Y, f([X+v']_{\rm c}))]
	\end{aligned}
\end{equation*}
Empirically, for class pair $a=(s, t)$ and $l$ being the 0-1 loss, the above problem can be reformulated as \cite{Post-TNNLS}:
\begin{equation}\label{eq:re_pert}
	\begin{aligned}
		& \underset{v}{\text{minimize}}
		& & ||v||_2\\
		& \text{subject to}
		& & \frac{1}{|{\mathcal D}_s|} \sum_{x\in{\mathcal D}_s} \mathds{1}[f([x+v]_{\rm c})=t]\geq \pi
	\end{aligned}
\end{equation}
where ${\mathcal D}_s$ is the subset of samples in ${\mathcal D}_{\rm c}$ from class $s$, $\mathds{1}[\cdot]$ is the indicator function (for counting the number of misclassifications from class $s$ to class $t$), and $\pi$ is a targeted misclassification fraction (which approximates one minus the Bayes error rate in practice).
Typically, $\pi$ is set large for a relatively large ``pair ASR'' assumed for a successful attack.
But an overly large $\pi$ may not be achievable for a backdoor class pair even with the actual trigger used by the attacker.
In practice, $\pi$ can be set large but not clearly larger than the ACC of the classifier to be inspected (which can be evaluated on the small dataset possessed by the defender).
The reasons are the following. 
For X2X attacks with an A2Ar setting, the ASR of a successful attack will not likely exceed the ACC.
For X2X attacks with other settings, the ASR of a successful attack may be larger than the ACC (even close to 100\%).
Since there is no prior knowledge about the attack setting, having a large $\pi$ without exceeding the ACC much will enlarge the probability for the trigger reverse-engineered for backdoor class pairs being close to the actual trigger used by the attacker.
Thus, in our experiments, we set $\pi=0.9$ for all datasets, which is sufficiently large without clearly exceeding the ACC of the classifiers to be inspected.

To solve \eqref{eq:re_pert} in practice, we minimize the following differentiable surrogate objective function using stochastic gradient descent \cite{Post-TNNLS}:
\begin{equation}\label{eq:re_pert_obj}
J_{st}^{\rm pert}(v) = -\frac{1}{|\mathcal{D}_s|}\sum_{x\in\mathcal{D}_s} p(t|[x + v]_c),
\end{equation}
with learning rate $10^{-4}$ and initial $v=0$. $p(t|x)$ denotes the classifier's posterior for class $t$ for arbitrary input $x\in{\mathcal X}$. The minimization of Eq. \eqref{eq:re_pert_obj} terminates when $\pi$ misclassification is achieved on $\mathcal{D}_s$.

\subsubsection{Reverse-Engineering Patch Triggers}\label{subsubsec:exp_main_supp_review_re_patch}

Patch triggers take the form $\delta(X)=(1-m)\odot X + m\odot u$, where $u$ is a small image patch, $m$ is a binary mask, and $\odot$ represents element-wise multiplication. For human imperceptibility, the patch size, which is solely determined by $m$, is usually small. Thus, the distance metric in problem \eqref{eq:re_main} can be specified by $d(X, \delta(X))=||X-\delta(X)||_0\approx||m||_0$. Accordingly, for each class pair $a=(s, t)$, we solve:
\begin{equation*}
	\begin{aligned}
		& \underset{\{m,u\}}{\text{minimize}}
		& & ||m||_0\\
		& \text{subject to}
		& & \{u,m\} \in \argmin_{\{u',m'\}} {\mathbb E}_{P_{XY|a}}[l(Y, f((1-m')\odot X + m'\odot u'))]
	\end{aligned}
\end{equation*}
Similarly, the problem above can be reformulated as the following:
\begin{equation}\label{eq:re_patch}
	\begin{aligned}
		& \underset{\{m,u\}}{\text{minimize}}
		& & ||m||_0\\
		& \text{subject to}
		& & \frac{1}{|{\mathcal D}_s|} \sum_{x\in{\mathcal D}_s} \mathds{1}[f((1-m)\odot x + m\odot u))=t]\geq \pi
	\end{aligned}
\end{equation}
Again, we set $\pi=0.9$ for all datasets considered in our experiments. Then, problem \eqref{eq:re_patch} can be solved by minimizing the surrogate objective function proposed by NC \cite{NC}:
\begin{equation}\label{eq:re_patch_obj}
J_{st}^{\rm patch}(u, m) = -\frac{1}{|\mathcal{D}_s|}\sum_{x\in\mathcal{D}_s} \log p(t|(1-m)\odot X + m\odot u)) + \lambda ||m||_1,
\end{equation}
where $\lambda$ is the Lagrange multiplier and the patch size is measured using the $\ell_1$ norm (instead of the $\ell_0$ norm in problem \eqref{eq:re_patch}) for differentiability. As suggested by \citeauthor{NC} (2019), the mask $m$ and the patch $u$ are both initialized to be image-wide and with initial values around 0.5 (for pixel values in [0, 1]) when minimizing Eq. \eqref{eq:re_patch_obj}. The multiplier $\lambda$ is adjusted based on whether the $\pi$ misclassification fraction from class $s$ to class $t$ (i.e. the constraint of problem \eqref{eq:re_patch} is achieved). More details about such adjustment and the learning rate can be found in the original implementation provided by \citeauthor{NC} (2019). To avoid poor local optimum when minimizing Eq. \eqref{eq:re_patch_obj}, we solve problem \eqref{eq:re_patch} for multiple times (e.g. 5 trials for CIFAR-10 and Imagenette and 3 trials for GTSRB), each with a randomly initialized $m$ and $u$. The solution with the minimum $||m||_1$ over all trials is deemed to be the reverse-engineered trigger.

\subsubsection{Reverse-Engineering Sample-Specific Triggers}\label{subsubsec:exp_main_supp_review_re_ss}

In fact, sample-specific triggers still use a {\it common} $\delta$, which may be as sophisticated as an autoencoder, for trigger embedding. The term ``sample-specific'' actually refers to that $\delta(x_i) - x_i$ and $\delta(x_j) - x_j$ are different for different samples $x_i$ and $x_j$. Unfortunately, accurate estimation of $\delta$ (e.g. estimating all the parameters of $\delta$ if it is an autoencoder) for a sample-specific trigger is still an open problem. But using the method proposed by \citeauthor{Post-TNNLS} (2020), we can estimate a simple additive perturbation in the intermediate layer of the classifier to approximate $\delta$. More specifically, suppose $f=f_2\circ f_1$ where $f_1:{\mathcal X}\rightarrow{\mathcal Z}$ maps an input to the intermediate feature space ${\mathcal Z}$ and $f_2:{\mathcal Z}\rightarrow{\mathcal Y}$ maps an intermediate feature to the output space ${\mathcal Y}$. For each class pair $a=(s, t)$, we solve:
\begin{equation*}\label{eq:re_inter}
	\begin{aligned}
		& \underset{w}{\text{minimize}}
		& & ||w||_2\\
		& \text{subject to}
		& & \frac{1}{|{\mathcal D}_s|} \sum_{x\in{\mathcal D}_s} \mathds{1}[f_2(f_1(x)+w)=t]\geq \pi
	\end{aligned}
\end{equation*}
by minimizing:
\begin{equation*}
J_{st}^{\rm inter}(w) = -\frac{1}{|\mathcal{D}_s|}\sum_{x\in\mathcal{D}_s} p'(t|f_1(x)+w),
\end{equation*}
using the same settings as for perturbation reverse-engineering in the input layer. Here, $p'(t|\cdot)$ denotes the posterior of class $t$ for intermediate features. In the experiments in Sec. \ref{subsec:exp_adaptive}, our UMD uses this technique to reverse-engineer the sample-specific trigger embedded by WaNet at the output layer of the first ``block'' of ResNet-18 (with four ``blocks'' in total) \cite{ResNet} and achieves excellent detection performance in both model inference and pair inference.

\subsection{Additional Results: Detection Performance of UMD with Different Choice of Confidence Level}\label{subsec:exp_main_supp_confidence}

In Sec. \ref{sec:exp} of the main paper, we showed the detection performance of UMD for a 95\% confidence level for a fair comparison with the SOTA baselines. Here, in Tab. \ref{tab:detection_model_confidence}, we show the model detection performance (via MIA) of UMD for a range of confidence levels from 0.6 to 0.999.
Clearly, more aggressive confidence thresholds (with a confidence level $<0.95\%$) slightly increase the true positive rate (i.e. an increment in MIA for classifiers being attacked) at the cost of a slight increment in the false positive rate (i.e. a decrement in MIA for benign classifiers).
On the other hand, more conservative thresholds (with a confidence level $>0.95\%$) slightly reduce the false positive rate, but the true positive rate is not affected much.
The results show that UMD prefers a more conservative confidence level since the attacks are typically associated with a large anomaly score if the putative backdoor class pairs are correctly selected.

\begin{table*}[h]
	\setlength{\tabcolsep}{4.5pt}
        \caption{MIA of our UMD for confidence levels (i.e. $1-\beta$) 0.6, 0.8, 0.9, 0.95, 0.99, and 0.999. Large confidence thresholds are helpful to reduce the false positive rate without much degradation in the true positive rate.}
	\centering{
\scalebox{0.83}{%
\begin{tabular}{cc|cccccccc}
    \multicolumn{7}{c}{(a) CIFAR-10}\\
\toprule \hline
\multicolumn{2}{l}{Setting} 
  &  benign & A2O & 2to2 & 5to5 & 8to8 & A2Ar\\
    \hline
    \multicolumn{2}{l}{$1-\beta= 0.6$} & 0.70 & 0.95 & 0.90 & 0.95 & 0.85 & 1.0\\
    \hline
   \multicolumn{2}{l}{$1-\beta= 0.8$}& 0.80 & 0.95 & 0.90 & 0.95 & 0.85  & 1.0\\\hline
    \multicolumn{2}{l}{$1-\beta= 0.9$} &  0.80 & 0.90 & 0.90 & 0.95 & 0.85 & 0.95\\
    \hline
    \multicolumn{2}{l}{$1-\beta= 0.95$} & 0.90 & 0.90 & 0.90 & 0.95 & 0.85 & 0.95 \\
    \hline
    \multicolumn{2}{l}{$1-\beta= 0.99$} & 0.90 & 0.90 &0.85 & 0.95&0.85 &0.95 \\
    \hline
    \multicolumn{2}{l}{$1-\beta= 0.999$} & 0.90 & 0.90 &0.85&0.95& 0.85&0.90 \\
    \hline \bottomrule
    \multicolumn{7}{c}{(b) GTSRB}\\
\toprule \hline
				\multicolumn{2}{l}{Setting} 
  &  benign & A2O & 20to20 & 30to30 & 40to40 & A2Ar\\
    \hline
    \multicolumn{2}{l}{$1-\beta= 0.6$} & 0.80 & 0.95 &0.80 & 0.90 &0.95 & 1.0\\
    \hline
   \multicolumn{2}{l}{$1-\beta= 0.8$} & 0.80 & 0.95 &0.80 &0.90 & 0.95& 1.0\\\hline
    \multicolumn{2}{l}{$1-\beta= 0.9$} & 0.80 & 0.95 &0.80 & 0.90 &0.90 & 1.0 \\
    \hline
    \multicolumn{2}{l}{$1-\beta= 0.95$} & 0.90 & 0.95 & 0.80 & 0.90 & 0.90 & 1.0 \\
    \hline
    \multicolumn{2}{l}{$1-\beta= 0.99$}  & 0.90 & 0.95 &0.80 &0.85&0.90 &0.95\\
    \hline
    \multicolumn{2}{l}{$1-\beta= 0.999$} & 0.90 & 0.95 &0.80 &0.85&0.90 &0.95  \\
    \hline \bottomrule
    \multicolumn{7}{c}{(c) ImageNette}\\
\toprule \hline
   \multicolumn{2}{l}{Setting} 
  &  benign & A2O & 3to3 & 5to5 & 8to8 & A2Ar\\
    \hline
    \multicolumn{2}{l}{$1-\beta= 0.6$} & 0.80 & 0.95 & 0.80 & 0.80 &0.90 &1.0 &\\
    \hline
   \multicolumn{2}{l}{$1-\beta= 0.8$} & 0.80 & 0.95 &0.80 &0.80 &0.90 &1.0 \\\hline
    \multicolumn{2}{l}{$1-\beta= 0.9$} & 0.80 & 0.95 &0.80 & 0.80 &0.80 & 1.0\\
    \hline
    \multicolumn{2}{l}{$1-\beta= 0.95$} & 0.80& 0.90 & 0.75 & 0.80 & 0.80 & 1.0 \\
    \hline
    \multicolumn{2}{l}{$1-\beta= 0.99$} & 0.80 & 0.90 &0.75 &0.75 &0.80&0.9\\
    \hline
    \multicolumn{2}{l}{$1-\beta= 0.999$} & 1.0 & 0.90&0.60 & 0.75 &0.70 &0.8 \\
    \hline \bottomrule
\end{tabular}
}
}
	\label{tab:detection_model_confidence}
\end{table*}

\newpage
\section{Supplementary of the Ablation Study}\label{sec:exp_ablation_supp}

\subsection{Baseline Variants of UMD}\label{subsec:exp_ablation_variants}

In this section, we provide details for the two baseline variants of UMD. The first baseline variant, $\text{UMD}^{\dag}$, directly uses the perturbation or patch size of the reverse-engineered trigger for each class pair for anomaly detection, without using our TR statistic. Since $\text{UMD}^{\dag}$ does not select a subset of putative backdoor class pairs, all the trigger statistics are used for the estimation of MAD, i.e.:
\begin{equation*}
\sigma^{\dag} = {\rm med}_{a\in{\mathcal Y}\times{\mathcal Y}\setminus{\mathcal B}}(|z_a^{-1} - {\rm med}_{a'\in{\mathcal Y}\times{\mathcal Y}\setminus{\mathcal B}} z_{a'}^{-1}|)
\end{equation*}
Note that the set ${\mathcal B}$ contains all class pairs with the same source class and target class -- trigger reverse-engineering is not performed for these class pairs. Then, an anomaly score is computed for each statistic by:
\begin{equation*}
r^{\dag}(z) = (z^{-1} - {\rm med}_{a\in{\mathcal Y}\times{\mathcal Y}\setminus{\mathcal B}} z_{a}^{-1}) / (1.4826 \cdot \sigma^{\dag})
\end{equation*}
If for any class pair $a$, the anomaly score $r^{\dag}(z_a)$ is larger than the confidence threshold determined by Eq. \eqref{eq:threshold}, we say that the classifier is attacked. And the $K$ class pairs with the largest anomaly score are detected as the backdoor class pairs. $K$ is the number of classes in the domain, which is also the largest number of class pairs that an X2X backdoor attack may involve.

The second baseline variant, $\text{UMD}^{\ddag}$, uses TR, but in a naive way. For each class pair $a$ {\it detected by} $\text{UMD}^{\dag}$, $\text{UMD}^{\ddag}$ performs a ``double check'' to see if the maximum ``mutual-transferability'' of $a$ with all the other class pairs, i.e. $\max_{a'\neq a}(T_{aa'}+T_{a'a})$, is in the top $K$ of all class pairs. If this is true, class pair $a$ is admitted as a backdoor class pair; otherwise, $a$ is deemed a non-backdoor class pair. Again, if there is at least one backdoor class pair being detected, the classifier is deemed to be attacked.

\subsection{Additional Results for UMD Compared with the Two Baseline Variants}\label{subsec:exp_ablation_variants_auc}

As shown in Sec. \ref{subsec:exp_ablation} of the main paper, the two baseline variants achieve an overly large false positive rate in model inference, though the confidence threshold is set for a 5\% false positive rate. Especially, the variant $\text{UMD}^{\dag}$ tends to predict any given classifier to be ``attacked''. In Tab. \ref{tab:ablation_others}, we show MIA for the two baseline variants and also our UMD for a range of confidence levels in $[0.95, 0.99, 0.999]$. Based on the results, even for extremely conservative confidence thresholds, the two baseline variants still have significantly high false positive rates (i.e. low MIA on benign classifiers). Indeed, an even larger confidence threshold may result in a meaningful false positive rate for the two variants while keeping a high true positive rate (i.e. a good separability between statistics for classifiers with and without attack), but such a threshold will be unknown to the defender {\it a priori} -- the defender will likely set a reasonable confidence level (e.g. near 95\%).

\begin{table}[h]
\centering
\caption{MIA for $\text{UMD}^{\dag}$, $\text{UMD}^{\ddag}$, and our UMD for confidence levels 0.95, 0.99, 0.999.}
        \begin{subtable}
	\centering{
		\scalebox{0.85}{%
			\begin{tabular}{c|ccccc}
				\toprule \hline
                    \multicolumn{6}{c}{$1-\beta= 0.95$}\\\hline
				& 2to2 & 5to5 & 8to8 & A2Ar & Benign\\  \hline
				$\text{UMD}^{\dag}$ & 1.0 & 1.0 &1.0  &1.0  &0 \\ \hline
                    $\text{UMD}^{\ddag}$ & 1.0 & 1.0 & 1.0 & 1.0 & 0.4 \\ \hline
				UMD & 1.0 & 1.0 & 0.9 & 0.9 & 0.9\\
                    \hline
			\end{tabular}
		}
	}
        \end{subtable}
        ~
        \begin{subtable}
	\centering{
		\scalebox{0.85}{%
			\begin{tabular}{c|ccccc}
				\toprule \hline
                    \multicolumn{6}{c}{$1-\beta= 0.99$}\\\hline
				& 2to2 & 5to5 & 8to8 & A2Ar & Benign\\  \hline
				$\text{UMD}^{\dag}$ & 1.0 & 1.0 &1.0  &1.0  &0 \\ \hline
                    $\text{UMD}^{\ddag}$ & 1.0 & 1.0 & 1.0 & 1.0 & 0.6 \\ \hline
				UMD & 1.0 & 1.0 & 0.9 & 0.9 & 0.9\\
                    \hline
			\end{tabular}
		}
	}
        \end{subtable}
        \begin{subtable}
	\centering{
		\scalebox{0.85}{%
			\begin{tabular}{c|ccccc}
				\hline
                    \multicolumn{6}{c}{$1-\beta= 0.999$}\\\hline
				& 2to2 & 5to5 & 8to8 & A2Ar & Benign\\  \hline
				$\text{UMD}^{\dag}$ & 1.0 & 1.0 &1.0  &1.0  &0 \\ \hline
                    $\text{UMD}^{\ddag}$ & 1.0 & 1.0 & 1.0 & 1.0 & 0.7 \\ \hline
				UMD & 0.9 & 1.0 & 0.9 & 0.9 & 0.9\\
                    \hline\bottomrule
			\end{tabular}
		}
	}
        \end{subtable}
 %        \begin{subtable}
	% \centering{
	% 	\scalebox{0.85}{%
	% 		\begin{tabular}{c|ccccc}
	% 			\hline
 %                    \multicolumn{6}{c}{$1-\beta= 0.9999$}\\\hline
	% 			No. pairs & 2 & 5 & 8 & 10 & Benign\\  \hline
	% 			$\text{UMD}^{\dag}$ & 1.0 & 1.0 &1.0  &1.0  &0.4 \\ \hline
 %                    $\text{UMD}^{\ddag}$ & 1.0 & 1.0 & 1.0 & 1.0 & 1.0 \\ \hline
	% 			UMD & 0.8 & 1.0 & 0.9 & 0.9 & 1.0\\
 %                    \hline\bottomrule
	% 		\end{tabular}
	% 	}
	% }
 %        \end{subtable}
	\label{tab:ablation_others}
\end{table}

\newpage
\section{Others}\label{sec:others}

\subsection{Collateral Damage for O2O Attacks}\label{subsec:others_o2o_collateral}

Since our UMD will always select at least two different class pairs for inference, we are interested in its detection capability against attacks with only one backdoor class pair, i.e. an O2O attack.
In Tab. \ref{tab:detection_advanced}, we show that UMD achieves relatively good performance against O2O attacks.
This is because, for O2O attacks, non-backdoor class pairs may suffer from collateral damage, such that samples from the source class will be misclassified to the target class when the backdoor trigger used by the attacker is embedded.
Thus, in addition to the true backdoor class pairs, there will exist effective ``backdoor'' class pairs that are not involved in the attack deliberately.
These class pairs typically share the same target class as the true backdoor class pair, since a relation between the backdoor trigger and the adversarial target class has been established when the classifier is trained on the poisoned training set.
In Fig. \ref{fig:collateral_damage}, for all ten O2O attacks, we show the histogram of the pair-based ASR (i.e. the fraction of samples from the source class being misclassified to the target class when the backdoor trigger is embedded) for all non-backdoor class pairs sharing the same target class as the true backdoor class pair.
There are several ``non-backdoor'' class pairs that have a pair-based ASR even larger than 80\%.

\begin{figure}[h]\label{fig:collateral_damage}
\centering
\includegraphics[width=.52\columnwidth]{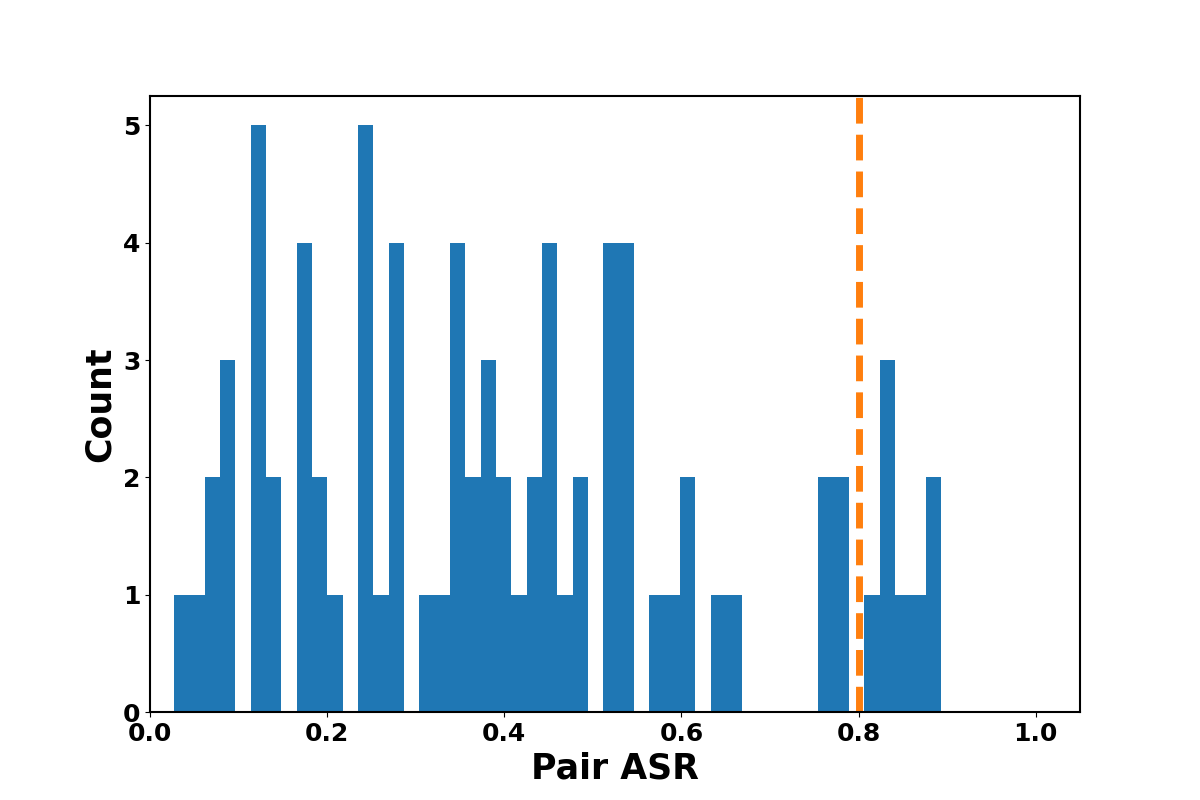}\caption{Collateral damage of O2O attacks. Some non-backdoor class pairs achieve as high as 80\% ASR due to the attack.}
\end{figure}

\subsection{Performance of UMD against X2X Attacks with Different Numbers of Images for Poisoning}\label{subsec:others_stength}

In this section, we show the performance of UMD against X2X attacks with different numbers of poisoning images. We train three groups of classifiers with the A2Ar setting and the perturbation trigger on CIFAR-10, but with 1500, 6000, and 10000 poisoning images respectively. As shown in Fig. \ref{fig:asr_num_image}, the average ASR of the attack grows with the number of poisoning samples, though still not exceeding the clean test accuracy, which is around 93\%. As shown in Tab. \ref{tab:num_poisoning_ablation}, UMD achieves generally stable detection performance for all these choices of the number of poisoning images. For attacks with 1500 poisoning images, we observe a drop in the average PDR (though still with a perfect MIA). This is because, with only 1500 images for poisoning, not every backdoor class pair achieves a sufficiently large ASR -- these class pairs are less distinguishable from non-backdoor class pairs than backdoor class pairs with high ASR, and thus are more difficult to detect.

\begin{figure}[h]\label{fig:asr_num_image}
\centering
\includegraphics[width=.52\columnwidth]{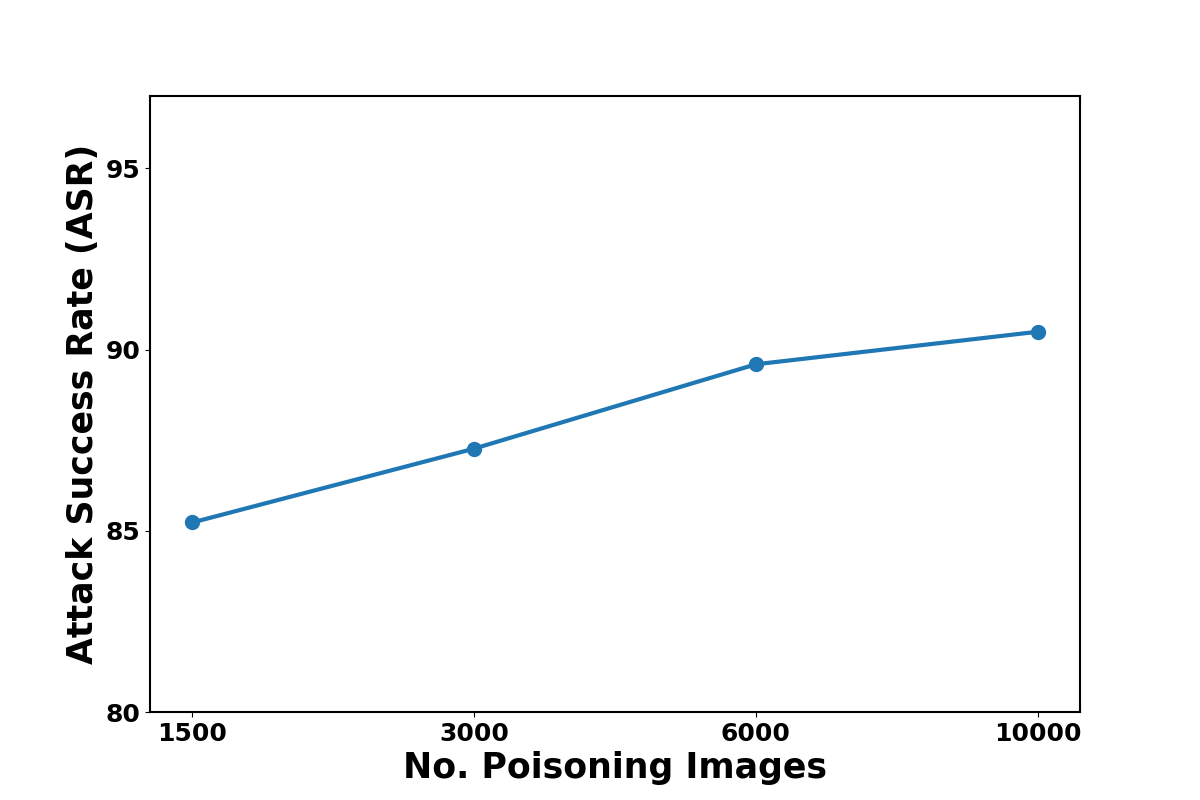}\caption{ASR for A2Ar attacks with different numbers of poisoning images on CIFAR-10.}
\end{figure}

\begin{table}[h]
	\setlength{\tabcolsep}{3.5pt}
        \caption{MIA and average PDR for UMD against A2Ar attacks on CIFAR-10 with 1500, 3000, 6000, and 10000 poisoned images. }
	\centering{
		\scalebox{1.0}{%
			\begin{tabular}{c|cccccccccc}
				\toprule \hline
				\multirow{2}*{No. images} & \multicolumn{2}{c}{1,500} & \multicolumn{2}{c}{3,000} & \multicolumn{2}{c}{6,000} & \multicolumn{2}{c}{10,000} &\\ 
                \cline{2-9}
                & MIA & PDR & MIA & PDR & MIA & PDR & MIA & PDR \\ \hline
				\textbf{UMD} & 1.0 & 0.70 & 0.90 & 0.92 & 1.0 & 0.84 & 0.90 & 0.92  \\
                \hline\bottomrule
			\end{tabular}
		}
	}
	\label{tab:num_poisoning_ablation}
\end{table}

% Zhen: show a plot for ASR, ACC vs No. images (attack perspective)

% Zhen: show MIA and PDR for the attacks (detection perspective)

\subsection{Intuition Behind the Objective Function in Problem (\ref{eq:clustering})}\label{subsec:others_intuition}

In Fig. \ref{fig:intuition_clustering}, we show a real TR-map for an A2A attack on CIFAR-10. Our clustering problem aims to find a ``core'' with the maximum ``brightness'' and an associated ``periphary'' with the maximum ``darkness'' \cite{core_periphery}.

\begin{figure}[t]\label{fig:intuition_clustering}
\centering
\includegraphics[width=.56\columnwidth]{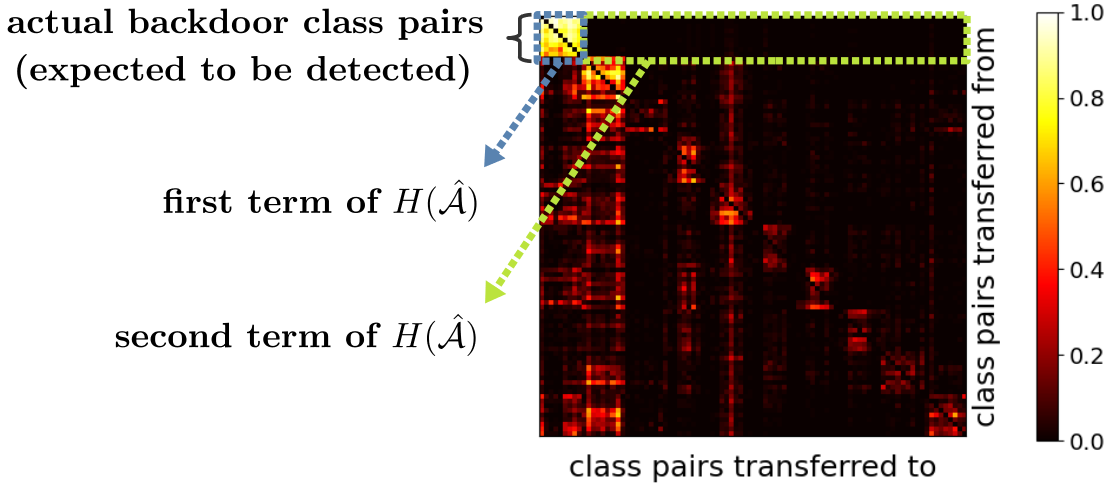}\caption{A real TR-map for an A2A attack on CIFAR-10. The orders of the 90 class pairs are the same for both axes.}
\end{figure}

\subsection{Computational Cost of UMD}\label{subsec:exp_computational_cost}

Empirically, each model inference on CIFAR-10, GTSRB, and Imagenette requires around 0.3h, 2.5h, and 4.3h, respectively, as measured on a single RTX 2080 Ti card.
As an off-line detection procedure, this time cost is acceptable compared with the training time on each dataset.
The main computational cost is induced by the need to determine for each of the $K(K-1)$ class pairs whether it is involved in a backdoor attack -- trigger reverse-engineering is performed for each class pair.
Since there is no constraint on the trigger reverse-engineering algorithm used by UMD, the efficiency of UMD can potentially be improved, e.g., by adopting the warm-up strategy by \citeauthor{Shen2021BackdoorSF} (2021) or the weighted-sum strategy by \citeauthor{L-RED} (2021) to accelerate the trigger reverse-engineering process.

\end{document}